\documentclass{article}

\usepackage{PRIMEarxiv}

\usepackage[utf8]{inputenc} % allow utf-8 input
\usepackage[T1]{fontenc}    % use 8-bit T1 fonts
\usepackage{hyperref}       % hyperlinks
\usepackage{url}            % simple URL typesetting
\usepackage{booktabs}       % professional-quality tables
\usepackage{amsfonts}       % blackboard math symbols
\usepackage{nicefrac}       % compact symbols for 1/2, etc.
\usepackage{microtype}      % microtypography
\usepackage{lipsum}
\usepackage{fancyhdr}       % header
\usepackage{graphicx}       % graphics
\graphicspath{{media/}}     % organize your images and other figures under media/ folder

\usepackage{amsmath,amsfonts,bm}

\def\1{\bm{1}}

\DeclareMathAlphabet{\mathsfit}{\encodingdefault}{\sfdefault}{m}{sl}
\SetMathAlphabet{\mathsfit}{bold}{\encodingdefault}{\sfdefault}{bx}{n}

\DeclareMathOperator*{\argmax}{arg\,max}

\usepackage{xcolor}         % colors
\usepackage{bm}
\usepackage{times}
\usepackage{soul}
\usepackage{amsmath,mathtools,amssymb}
\usepackage{amsthm}
\usepackage{natbib}

\title{A template for Arxiv Style
\thanks{\textit{\underline{Citation}}: 
\textbf{Authors. Title. Pages.... DOI:000000/11111.}} 
}

\author{
  Author1, Author2 \\
  Affiliation \\
  Univ \\
  City\\
  \texttt{\{Author1, Author2\}email@email} \\
   \And
  Author3 \\
  Affiliation \\
  Univ \\
  City\\
  \texttt{email@email} \\
}

\makeatletter
\newtheorem*{rep@theorem}{\rep@title}
\newcommand{\newreptheorem}[2]{%
\newenvironment{rep#1}[1]{%
 \def\rep@title{#2 \ref{##1}}%
 \begin{rep@theorem}}%
 {\end{rep@theorem}}}
\makeatother

\newtheorem{theorem}{Theorem}
\newtheorem{lemma}{Lemma}
\newtheorem{claim}{Claim}
\newtheorem{corollary}{Corollary}
\newtheorem{remark}{Remark}

\newreptheorem{theorem}{Theorem}
\newreptheorem{lemma}{Lemma}

\usepackage{algorithm}
\usepackage{algpseudocode}
\let\oldComment=\Comment
\renewcommand{\Comment}[1]{\oldComment{\texttt{#1}}}
\algnewcommand{\LeftComment}[1]{\Statex $\triangleright$ \texttt{#1}}
\algnewcommand{\RightComment}[1]{\Statex \leavevmode\hfill$\triangleright$ \texttt{#1}}
\algnewcommand\algorithmicinput{\textbf{Input:}}
\algnewcommand\Input{\item[\algorithmicinput]}%
\algnewcommand\algorithmicoutput{\textbf{Output:}}
\algnewcommand\Output{\item[\algorithmicoutput]}%
\algnewcommand\algorithmicinitial{\textbf{Initialize:}}
\algnewcommand\Initial{\item[\algorithmicinitial]}%

\usepackage{subcaption}
\usepackage{bbm}

\definecolor{mydarkblue}{rgb}{0,0.08,0.45}
\hypersetup{ %
    pdftitle={},
    pdfsubject={},
    pdfkeywords={},
    pdfborder=0 0 0,
    pdfpagemode=UseNone,
    colorlinks=true,
    linkcolor=mydarkblue,
    citecolor=mydarkblue,
    filecolor=mydarkblue,
    urlcolor=mydarkblue,
}
\usepackage[colorinlistoftodos,prependcaption,textsize=tiny,textwidth=1cm,color=green]{todonotes}

\newcommand{\compilehidecomments}{false}%HIDE comments
\ifthenelse{ \equal{\compilehidecomments}{true} }{%
    \newcommand{\jz}[1]{}
    \newcommand{\xc}[1]{}
    \newcommand{\mo}[1]{}
    \newcommand{\fg}[1]{}
    \newcommand{\rev}[1]{}
}
{
    \newcommand{\jz}[1]{{\color{orange} [\text{JZ:} #1]}}
    
    \newcommand{\mo}[1]{{\color{blue} [\text{MH:} #1]}}
    \newcommand{\fg}[1]{{\color{purple} [\text{FG:} #1]}}
    \newcommand{\rev}[1]{{\color{red}[#1]}}
}

\usepackage{xspace}

\newcommand{\DRAA}{\texttt{DRAA}\xspace}

\newcommand{\barbar}{\texttt{BARBAR}\xspace}
\newcommand{\MAB}{\texttt{MAB}\xspace}
\newcommand{\MAtB}{\texttt{MA2B}\xspace}
\newcommand{\Reg}{\text{Reg}\xspace}
\usepackage{letltxmacro}
\usepackage{fontawesome5}

\DeclarePairedDelimiter\abs{\lvert}{\rvert}

\renewcommand{\ge}{\geqslant}
\renewcommand{\le}{\leqslant}
\renewcommand{\geq}{\geqslant}

\title{Multi-Agent Stochastic Bandits \\ Robust to Adversarial Corruptions}

\author{Fatemeh Ghaffari\\
CICS, UMass Amherst\\
Amherst, MA 01003\\
\texttt{fghaffari@cs.umass.edu} 
\And
Xuchuang Wang\\
CICS, UMass Amherst\\
Amherst, MA 01003\\
\texttt{xuchuangwang@cs.umass.edu} 
\And
Jinhang Zuo \\
CS, CityU \\
Kowloon, Hong Kong\\
\texttt{jinhangzuo@gmail.com}
\And
Mohammad Hajiesmaili \\
CICS, UMass Amherst\\
Amherst, MA 01003\\
\texttt{hajiesmaili@cs.umass.edu}
}

\begin{document}

\maketitle

\begin{abstract}
    We study the problem of multi-agent multi-armed bandits with adversarial corruption in a heterogeneous setting, where each agent accesses a subset of arms. The adversary can corrupt the reward observations for all agents. Agents share these corrupted rewards with each other, and the objective is to maximize the cumulative total reward of all agents (and not be misled by the adversary).
    % despite the presence of corruptions.
    % We propose a corruption-robust learning algorithm, \texttt{distbarbar}, a fully distributed approach. 
    We propose a multi-agent cooperative learning algorithm that is robust to adversarial corruptions.
    For this newly devised algorithm, we demonstrate that an adversary with an unknown corruption budget \(C\) only incurs an additive \(O((L / L_{\min}) C)\) term to the standard regret of the model in non-corruption settings, where \(L\) is the total number of agents, and \(L_{\min}\) is the minimum number of agents with mutual access to an arm.
    % is quantifiably bounded, introducing a corruption term of \(O((L_{\max} / L_{\min}) C)\) to the standard regret in non-corruption settings, where \(L_{\max}\) is the maximum number of agents accessing the same arm, \(L_{\min}\) is the minimum number of agents accessing the same arm, and \(C\) is the total corruption level and unknown a priori. 
    As a side-product,
    our algorithm also improves the state-of-the-art regret bounds when reducing to both the single-agent and homogeneous multi-agent scenarios, tightening multiplicative \(K\) (the number of arms) and \(L\) (the number of agents) factors, respectively.
    % we obtain a regret bound that matches the lower bound.
    % Our algorithm improves upon previous work for single-agent and homogeneous multi-agent cases, matching the lower bound. 
    % Additionally, it maintains a communication cost of \(O(L \log T)\) across \(T\) decision rounds, where \(L\) is the number of agents.
    
    % Lastly, we also corroborate the efficiency of our proposed algorithm via numerical experiments.
\end{abstract}

\section{Introduction}\label{sec:1}

% The multi-armed bandit (\texttt{MAB}) \cite{bubeck2012regret} stands as a foundational problem in online learning. In stochastic \texttt{MAB}, an agent has access to a set of arms \(\mathcal{K} \coloneqq \{1, 2, ..., K\}\), \(K\in \mathbb{N}^+\). Each arm is associated with a reward distribution unknown to the agent. 
% % The agent lacks information about the reward distributions, and 
% The objective is to maximize the total collected reward over \(T\in \mathbb{N}^+\) rounds. This is equivalent to minimizing regret, defined as the difference between the cumulative rewards of continuously pulling the best arm and the actual total rewards obtained by the agent in \(T\) rounds.

Motivated by its broad applicability in large-scale learning systems, the multi-agent multi-armed bandits (\texttt{MA2B}) problem has recently been studied in various settings, e.g., cooperative \texttt{MA2B} \citep{vial2021robust,liu2021cooperative,hillel2013distributed,yang2021cooperative}, competitive \texttt{MA2B} \citep{liu2010distributed,wang2020optimal,mehrabian2020practical}, and \texttt{MA2B} with collision~\citep{liu2010distributed,wang2020optimal,shi2021heterogeneous}, etc. 
Each agent can access arms in this multi-agent scenario, solving an \texttt{MAB} problem. 
The goal for each agent is to maximize the cumulative reward, or in other words, to minimize the aggregated regret, which is the difference between the cumulative rewards of the optimal arm and the agent's choices. Agents also aim to minimize their communication cost. Previous works typically assume a homogeneous setting, where all agents can access the entire set of arms and solve the same \texttt{MAB} instance.

% In this scenario, multiple agents (\(\mathcal{L} \coloneqq \{1, 2, ..., L\}\) \(L\in\mathbb{N}^+\)), each have access to a subset of arms \(\mathcal{K} \coloneqq \{1, 2, ..., K\}\), \(K\in \mathbb{N}^+\). Each agent solve an instance of a \texttt{MAB} problem. The goal is to minimize aggregated regret and communication costs across all agents. Previous works typically assume a homogeneous setting, where all agents have access to the entire set of arms and solve the same \texttt{MAB} instance.

In many real-world applications, these agents operate in varying but possibly overlapping contexts and face adversarial corruption. Examples include recommender systems~\citep{zhou2019privacy,immorlica2019bayesian,silva2022multi}, online advertising~\citep{chu2011contextual}, and shortest path routing~\citep{zou2014online,talebi2017stochastic}. In recommender systems, for example, products (arms) have different ratings (rewards),
% , which correspond to the expected reward of each arm. 
and platform operators (agents) are associated with distinct geographical jurisdictions, where the set of available products, as well as the rules governing their sale may differ. This results in each agent having access to a different set of arms, possibly overlapping with other agents' arm-sets. Each agent aims to maximize user satisfaction by recommending products within its jurisdiction. An agent ``pulls'' an arm by recommending a product to a user and then observes the reward as the user's rating. However, fake ratings and reviews (corruption) may be introduced by merchants to manipulate their product's ratings or degrade competitors' products. Agents are unaware of the presence or level of this corruption. While they benefit from sharing information, privacy requirements necessitate that they limit communication frequency. The objective, therefore, is to develop an algorithm that minimizes the impact of adversarial corruption with a low communication cost.

In this paper, we formally model these realistic problems as heterogeneous \texttt{M2AB} with adversarial corruptions, where \(L\in\mathbb{N}^+\) heterogeneous agents cooperatively play a stochastic \texttt{MAB} game with \(K\in\mathbb{N}^+\) arms under corruptions.
Each arm \(k\) is associated with a reward distribution. 
Each agent \(\ell\) can access a subset of arms, called the \emph{local arm-set}.
In each decision round, each agent picks one arm from its local arm-set to pull and observes a reward sample drawn from the pulled arm, which an adversary possibly corrupts. 
Similar to the regular \texttt{MAB} setting, each agent aims to maximize the reward over a horizon of \(T\in\mathbb{N}^+\) rounds. 
We assume that all agents can communicate with each other in a fully connected setting to share information and expedite the learning process. The formal definition of the problem setting is provided in Section~\ref{sec:2}.

\paragraph{Technique challenge overview.}
We aim to study cooperative algorithms for the heterogeneous \texttt{MA2B} robust to adversarial corruptions.
There is limited prior research on multi-agent online learning under corruptions, and most devoted to homogeneous settings~~\citep{vial2022robust,mitra2021robust,mahesh2022multi}, and heterogeneous multi-agent scenario~\citep{yang2022distributed,wang2023explore} of more practical interest is not explored. Especially as the robustness requires each agent to utilize individual randomness in their decisions, agents would have uncoordinated actions, which is different from the standard cooperative bandit algorithms~\citep{wang2023achieve,wang2020optimal} where agents often coordinate with each other to pull arms.
Besides the cooperation challenge of multiple agents due to the randomization,
agents' heterogeneity constrains them from pulling only their local arms, which may overlap with those of other agents. This overlap complicates cooperation since each agent's local optimal arm and corresponding pull probability may differ.

In this paper, we propose a fully distributed algorithm for the heterogenous \texttt{MA2B} that is agnostic to the corruption level. The algorithm is epoch-based with doubling epoch lengths, where each agent's arm pull probabilities are set at the start of the epoch and remain constant throughout. Since agents communicate only at the end of each epoch, the communication cost is logarithmic. Before each epoch, the local arm-set is divided into active and bad arms, allocating fewer resources to the bad arms. At the end of the epoch, each agent calculates the empirical mean for each arm based on both observed and shared rewards from other agents. To address the differences in local arm-set sizes, we apply a weighted averaging technique to mitigate the effects of heterogeneity. Our algorithm improves upon the algorithm \barbar presented in \cite{gupta2019better} and the algorithm proposed by \cite{liu2021cooperative} in the special cases of single-agent \MAB and homogeneous \MAtB, removing the \(K\) and \(L\) factors from the corruption term in the regret upper bound and matching the lower bound of \(O(C)\) established in \cite{gupta2019better} and \cite{lykouris2018stochastic} for the single-agent case, and in~\cite{liu2021cooperative} for the homogeneous multi-agent case.

\subsection{Contributions}
% In this paper, we propose robust algorithms for heterogeneous \texttt{MA2B} in the presence of adversarial corruption, and, specifically, study two kinds of cooperations: leader-follower based \cite{wang2020optimal}, and fully distributed \cite{yang2022distributed}.
% In leader-follower (\texttt{LF}) algorithms, only a single leader agent makes decisions and assigns arms to all followers to pull, while in the fully distributed algorithms, all agents make their own decisions cooperatively.
% It is crucial to emphasize that these two cooperation models are inherently distinct, and developing corruption-robust cooperative algorithms for each scenario requires addressing unique challenges.

% \paragraph{New model}
% In Section~\ref{sec:2}, we introduce a heterogeneous \MAtB model with adversarial corruption. We extend the standard definitions of corruption level and regret upper bound to fit this setting.

\paragraph{Algorithm design (Section~\ref{sec:3})}
We introduce the \DRAA algorithm, a fully distributed, robust approach for heterogeneous \texttt{MA2B} that is agnostic to the corruption level, where agents make independent decisions while cooperating. 
This is the first work to address stochastic \MAtB with adversarial corruption in a heterogeneous setting. We introduce new algorithmic techniques to address the challenges posed by heterogeneity.
The algorithm incorporates a set-splitting technique to classify arms as active or bad, allocating some resources to bad arms. We also apply a weighted averaging method to estimate the empirical mean for each arm to properly interpret information receive from other agents. 
% Our algorithm improves upon \barbar and \cite{liu2021cooperative}'s in the special cases of single-agent \MAB and homogeneous \MAtB with corrupted rewards.
Our theoretical results demonstrate that, in the special case of single-agent \MAB, our algorithm improves upon the results of \barbar. In the case of homogeneous \MAtB, it surpasses the regret upper bound established by \cite{liu2021cooperative}. 
% Numerical experiments further confirm the superiority of our approach.

\paragraph{Theoretical analyses (Section~\ref{sec:4})} We show that \DRAA achieves a regret of \(O((L / L_{\min}) C) + \tilde{O}\left({\log (T)}K / {\Delta_{\min}}\right)\), where \(L_{\min}\) is the minimum number of agents with mutual access to the same arm, \(\Delta_{\min}\) is the minimum non-zero gap between the expected reward of any arm \(k\) and any local best arm, and \(C\) is the total corruption level. We also reduce the algorithm to single-agent \MAB, achieving a regret bound of \(O(C) + \tilde{O}\left({\log (T)}K / {\Delta_{\min}}\right)\). This improves upon \barbar by eliminating the \(K\) factor from the corruption term. Similarly, by reducing our setting to a homogeneous \MAtB, we achieve a regret bound of \(O(C) + \tilde{O}\left({\log (T)}K / {\Delta_{\min}}\right)\),
thus removing the \(L\) term from the corruption term compared to \cite{liu2021cooperative}. Our algorithm also achieves a near-optimal regret bound in the absence of corruption, while maintaining a logarithmic communication complexity. 
\subsection{Related Work}

\paragraph{Single-agent \MAB} The \texttt{MAB} problem is typically studied in two different settings: stochastic or adversarial. In the stochastic setting, realized rewards are sampled from reward distributions. In the adversarial setting, an adversary manipulates the reward either at the beginning of the game or at every pull. However, real-life situations often do not neatly fit into one of these two categories. 
For instance, as explained in the example, in a recommender system observed rewards are product ratings, some of which may be fake, while the others are genuine~\citep{kapoor2019corruption}. In online shortest path routing for large networks, the distribution of links can be stochastic, adversarial, or a combination of both~\citep{zhou2019toward,yang2020adversarial}. Therefore, it is more practical to consider a combination of these settings, where adjusting a parameter can transition the setting between stochastic, adversarial, or some partial state in between. Towards this end, we consider stochastic \texttt{MA2B} in the presence of adversarial corruption~\citep{lykouris2018stochastic,gupta2019better,zhao2021linear,wei2022model,bogunovic2021stochastic}. In this scenario, rewards are first generated from underlying distributions, these stochastic rewards are then subject to an adversary's corruption, and only the corrupted rewards are observed by the agents. Additionally, agents have no information about the extent of corruption in the observed reward. Further details on corruption are provided in Section~\ref{sec:2}.

\paragraph{Multi-Agent \texttt{MA2B}}
Among the vast amount of \texttt{MA2B} literature~\citep{chawla2023collaborative,hillel2013distributed,immorlica2019bayesian,liu2010distributed,liu2021cooperative,mehrabian2020practical,shi2021heterogeneous,vial2021robust,vial2022robust,wang2020optimal}, the work of \cite{yang2022distributed}---which first introduces the heterogeneous \texttt{MA2B} model---is the most related to ours.
\cite{yang2022distributed} extend the Upper Confidence Bound (\texttt{UCB}) and the Active Arm Elimination (\texttt{AAE})~\citep{even2006action} algorithms to a heterogeneous setting. Their Cooperative UCB (\texttt{CO-UCB}) algorithm achieves a regret of \(O(\sum_{k: \Tilde{\Delta}_k<0}\log{T}/\Tilde{\Delta}_k)\), where \(\Tilde{\Delta}_k\) is defined as the minimum gap between the mean reward of arm \(k\) and the local optimal arm in the local arm-sets of agents which have access to \(k\).  
% In Section~\ref{sec:5}, we implement the (non-robust) \texttt{CO-UCB} algorithm from \cite{yang2022distributed} to compare with the performance of our proposed algorithms. 
In the heterogeneous \texttt{MA2B} problem \citep{yang2022distributed,chawla2023collaborative}, each agent \(\ell\) has access to a subset of arms. This subset is referred to as the local arm-set for agent \(\ell\). Similar to the regular \texttt{MAB} setting, each agent aims to maximize the reward within its own local arm-set. We assume that all agents can communicate with each other in a fully connected setting to share information and expedite the learning process. The formal definition of the problem setting is provided in Section~\ref{sec:2}.

\paragraph{\texttt{MAB} with corruption} The notion of corruption was first introduced in \cite{lykouris2018stochastic}. Their proposed algorithm, Multi-Layer Active Arm Elimination Race, exhibits a regret of \(\Tilde O(\sum_{i \neq i^*} C/\Delta_k)\), where \(C\) is the corruption level and \(\Delta_k\) represents the gap between the mean reward of arm \(k\) and the optimal arm. 
% Also the notation \(\Tilde{O}(.)\) signifies an upper bound up to a logarithmic factor. 
This regret bound is subsequently enhanced by \cite{gupta2019better}. Their algorithm, Bandit Algorithm with Robustness: Bad Arms Get Resource (\barbar), achieves a regret of \(O(KC) + \Tilde{O}(\sum_{k \neq k^*}1/\Delta_k)\), introducing an additive term accounting for corruption. \barbar improves upon the Active Arm Elimination (\texttt{AAE}) algorithm \cite{even2006action} where instead of permanently eliminating suboptimal arms from the active arm-set, they are pulled with reduced probability. \cite{zimmert2021tsallis} have also proposed a best-of-both-worlds algorithm that can also be applied to address the bandits with corruption setting.
The corruption setting has also been explored in linear contextual bandit settings by \cite{zhao2021linear,wei2022model,bogunovic2021stochastic}. All the aforementioned works analyze corruption within a single-agent setting. Extending any of these robust algorithms to a heterogeneous \texttt{MA2B} setting presents unique challenges due to differences in the local arm-sets and optimal arms among agents. We discuss these challenges in more detail in Section~\ref{sec:4}.

\paragraph{Corruption in the \texttt{MA2B} setting}
The first study to introduce the concept of corruption in a \texttt{MA2B} setting is by \cite{liu2021cooperative}. They extend the corruption level definition from \cite{lykouris2018stochastic,gupta2019better} to a homogeneous multi-agent setting, utilizing the same total corruption definition formally outlined in \eqref{eq:corr}. The authors achieve a high probability regret bound of \(O(LC) + \Tilde{O}(K/\Delta_{\text{min}})\). The key distinction between their work and ours lies in their consideration of a homogeneous \texttt{MA2B} setting. They extend the \barbar algorithm from \cite{gupta2019better} to a leader-follower algorithm, where the leader agent allocates subsets of arms to the follower agents. Subsequently, each follower pulls arms from their assigned set based on a probability calculated from the agent reward gap. This ensures that agents with higher observed rewards are more likely to be pulled again. Extending this algorithm to a heterogeneous setting introduces several technical challenges, particularly in arm assignment based on each agent's local arm-set.
Lastly, we distinguish our work from another line of research, as exemplified in \cite{vial2021robust,vial2022robust}, which aims to develop robust algorithms in the presence of malicious agents rather than corrupted reward observations in our setting.

\section{Problem Setting}\label{sec:2}
% In this section, we introduce the problem of multi-agent multi-armed bandits with heterogeneous agents under adversarial corruptions.

\paragraph{Multi-Agent multi-armed bandit (\texttt{MA2B}) model} 
A \texttt{MA2B} model consists of \(L\in\mathbb{N}^+\) agents in set $\mathcal{L} \coloneqq \{1, 2, \dots, L\}$ and \(K\in\mathbb{N}^+\) arms in set $\mathcal{K} \coloneqq \{1, 2, \dots, K\}$, where each agent \(\ell\in\mathcal{L}\) has access to a subset of arms $\mathcal{K}_{\ell} \subseteq \mathcal{K}$ with size $K_{\ell} \coloneqq\lvert{\mathcal{K}_\ell}\rvert$, which can be \emph{heterogeneous}. Each arm $k \in \mathcal{K}$ is associated with an i.i.d. random reward bounded in $[0,1]$ with reward mean $\mu_k$. We will refer to the distribution of arm \(k\) as \(\mathcal{D}(\mu_k)\). The local best arm for agent \(\ell\), denoted by $k_{\ell}^*$, is the arm with the highest reward mean in $\mathcal{K}_{\ell}$, i.e., 
\(k_\ell^* \coloneqq \argmax_{k \in \mathcal{K}_{\ell}} \mu_{K}\). For any other arm $k \in \mathcal{K}_{\ell}\setminus \{k_{\ell}^*\}$, we define the local reward gap as $\Delta_{k, \ell} \coloneqq \mu_{k_{\ell}^*} - \mu_k$. We also define \(\mathcal{L}_k \subseteq  \mathcal{L}\) as the set of agents with access to arm \(k\), i.e., \(\mathcal{L}_k \coloneqq \{\ell \in \mathcal{L}: k\in\mathcal{K}_\ell\}\), and we set $L_k$ as the size of $\mathcal{L}_k$.
Denote \(T\in\mathbb{N}^+\) as the total number of decision rounds, and set \(\mathcal{T}\coloneqq \{1,2,\dots,T\}\).

\paragraph{Corruption mechanism} 
In each round \(t\), stochastic rewards \(r_{k, \ell}^t\) are drawn for each agent \(\ell\) and arm \(k\) from the corresponding reward distribution of the arm (Line~\ref{line:draw_rewards}). An adversary then observes these realized rewards as well as the historical actions and rewards of each agent (Line~\ref{line:adv_observe}). Based on this information, the adversary returns corrupted rewards \(\Tilde{r}_{k, \ell}^t \in [0, 1]\) for each arm \(k\) across all agents \(\ell\) (Line~\ref{line:corruption}). Each agent \(\ell\) then selects an arm \(k_{\ell}^t\) to pull and observes only the corrupted reward \(\Tilde{r}_{k_{\ell}^t, \ell}^t\) (Line~\ref{line:arm_pull_proc}). The corrupted reward observation process is presented in Procedure~\ref{proc:corruption}.

\floatname{algorithm}{Procedure}
\begin{algorithm}[tp]
    \caption{Reward generation and corruption procedure}\label{proc:corruption}
    \begin{algorithmic}[1]
    \For {\(t = 1, 2, \dots, T\)}
        \State Draw stochastic rewards \(r_{k, \ell}^t \sim \mathcal{D}(\mu_k)\) for all \(\ell \in \mathcal{L}\) and \(k \in \mathcal{K}_\ell\).
        \label{line:draw_rewards}
        \State Adversary observes realized rewards \(r_{k, \ell}^t\), 
        \Statex
        \hfill as well as the rewards and actions of each agent in previous rounds. 
        \label{line:adv_observe}
        \State Adversary returns the corrupted reward \(\Tilde{{r}}_{k, \ell}^t \in [0, 1]\) for all arms \(k\) and all agents \(\ell\).
        \label{line:corruption}
        \State Each agent \(\ell\) pulls an arm \(k_{\ell}^t\), and observes the corrupted reward \(\Tilde{{r}}_{k_{\ell}^t, \ell}^t\).
        \label{line:arm_pull_proc}
    \EndFor
    \end{algorithmic}
\end{algorithm}
\floatname{algorithm}{Algorithm}
% \begin{itemize}
%     \item For each agent $\ell \in \mathcal{L}$ and each arm $k \in \mathcal{K}_{\ell}$, there is a stochastic reward generated at round $t$ from the arm's reward distribution. This reward is denoted as $r_{k, \ell}^t$, and they form the vector $\bm{r}_{\ell}^t = (r_{k, \ell}^t)_{k\in\mathcal{K}_{\ell}} \in [0, 1]^{K_{\ell}}$.
%     \item The adversary observes $\bm{r}_{\ell}^t$ for all agents, and returns a corrupted reward vector, $\Tilde{\bm{r}}_{\ell}^t \in [0, 1]^{K_{\ell}}$.
%     \item Each agent $\ell \in \mathcal{L}$ pulls an arm $k_{\ell}^t$, and observes the corrupted reward $\Tilde{\bm{r}}_{k_{\ell}^t, \ell}^t$, generated by the adversary in the previous step.
% \end{itemize}
The total corruption level \(C\) across all agents is commonly defined as the summation of all agents' individual corruption level \(C_\ell\) as follows,
\[
\label{eq:corr}
    C \coloneqq \sum_{\ell=1}^L C_{\ell}, \text{ where }
    C_{\ell} \coloneqq \sum_{t = 1}^T \left\lVert \tilde{r}_{\ell}^t - r_{\ell}^t\right\rVert_{\infty}.
\]
Here, the infinity norm for individual corruption \(C_\ell\) is employed to compute the corruption level, capturing the worst-case scenario where agents consistently pull the arm with the highest corruption in each round. This infinity norm aligns with our objective of designing algorithms that are agnostic to the corruption level, and it is a standard definition in prior literature of bandits with corruptions, e.g.,~\citet{gupta2019better,lykouris2018stochastic}.

% This notation is also used in~\cite{gupta2019better} to calculate the regret upper bound of \barbar.

\paragraph{Multi-Agent communication}
Agents cooperate via gathering information from each other, which incurs communication costs~\citep{wang2022achieving, yang2022distributed}. 
For simplicity, we assume broadcast communications, and the total cost is defined as follows,
\[
\text{Comm}(T) \coloneqq \mathbb{E}\left[
\sum_{\ell \in \mathcal{L}}\sum_{t \in\mathcal{T}}\mathbbm{1}\left\{\text{Agent \(\ell\) broadcasts a message in \(t\)}\right\}\right].
\]
In a distributed scenario, all the agents follow the same communication scheme, and messages typically include the agent's pulled arm indices and observed rewards from the previous rounds.

\paragraph{Regret objective} Each agent seeks to maximize its cumulative reward by frequently pulling its locally optimal arm. This objective is equivalent to minimizing regret, a standard metric in \texttt{MAB} problems. Regret for each agent \(\ell\in\mathcal{L}\) is defined as the difference between the accumulative rewards of the pulled arm by the concerned algorithm and the accumulative rewards of pulling the locally optimal arm over \(T\) rounds as follows,
\begin{equation}
    \Reg_{\ell}(T) \coloneqq
    \mathbb{E} \left[ \sum_{t = 1}^T (\mu_{k^*_{\ell}} - \mu_{k_{\ell}^t}) \right],
\end{equation}
where $k^*_{\ell}$ is the local optimal arm for agent \(\ell\), \(\mu_{k_{\ell}^t}\) is the mean reward of the arm played by agent \(\ell\) in round \(t\), and the expectation is taken over the randomness of the learning algorithm.
The total regret of all agents is then defined as the sum of the regrets of all agents: \(\Reg(T) \coloneqq \sum_{\ell} \Reg_{\ell}(T).\)
We aim to bound the total regret in terms of the total corruption level $C$.
\section{Algorithms}\label{sec:3}
In this section, we introduce the Distributed Robust Arm Activation (\DRAA) algorithm to defense against the adversarial corruptions for the heterogeneous \MAtB model.

\paragraph{Main algorithmic idea}
This \DRAA algorithm operates in epochs, each doubling its prior size.  
Each epoch \(m\) is divided into three phases:
active arm-set construction, arm-pulling and communication, and estimate updates. 
In the first phase --- \emph{active arm-set construction} --- each agent \(\ell\) constructs an active arm-set \(\mathcal{A}_\ell^m\), including all ``good'' local arms \(k\in\mathcal{K}_\ell\), and its complement, bad arm-set \(\mathcal{B}_\ell^m = \mathcal{K}_\ell \setminus \mathcal{A}_\ell^m\). 
Then, agents assign the arm pull probabilities to arms according to whether it is in the active arm-set \(\mathcal{A}_\ell^m\) or the bad arm-set \(\mathcal{B}_\ell^m\), respectively.
The second phase --- \emph{arm-pulling and communication} --- involves agents pulling arms according to the assigned probabilities in all rounds of the epoch \(m\), followed by communicating with other agents.
In the final phase of \emph{reward estimation}, each agent estimates the total reward for each arm using its local observed rewards and those received from others, and then calculates each arm's empirical reward gaps, determining the arm pull probabilities for the next epoch.
In Section~\ref{sec:3.1}, we will examine each part of the algorithm in detail.

\subsection{Algorithmic detail}
\label{sec:3.1}
\begin{algorithm}[tp]
% RObust Multi-Epoch Set Splitting (RoMESS)
% (Multi agent) Distributed Robust arm activation (DRAA)
    \caption{Distributed Robust Arm Activation (\DRAA) for agent \(\ell\)}
    \label{alg:alg}
    \textbf{Input}: full arm-set \(\mathcal{K}\), the local arm \(\mathcal{K}_\ell\) of all agents, time horizon \(T\), small probability \(\delta\).
    \\
    \textbf{Initialization}: \(\Delta_{k, \ell}^0 \gets 1\), and \(r^m_{k, \ell} \gets 1\) for all \(k \in \mathcal{K}_{\ell}\), and \(\lambda \gets 2^{24} \log{(8KL\log{(T)}/\delta)}\).
    \\
    \vspace{-11pt}
    \begin{algorithmic}[1] %[1] enables line numbers
        \For {epochs \(m \gets 1, 2, \dots, M\)}
        \LeftComment{Phase 1: Active Arm Constructions and Probability Calculation}
        \State \(T^m\gets \lambda K 2^{2(m - 1)} / L_{\min}\).
        \label{line:T_m}
        \State \(\mathcal{A}_\ell^{m + 1} \gets \left\{k\in \mathcal{K}_\ell: r_{\max, \ell}^{m} - r_{k,\ell}^{m} < \frac{1}{2^{m + 3}}\sqrt{\frac{L_{\min}}{L}}-3 \times 2^{-7}\right\}\).
        \label{line:set_A} \Comment{Active arm-set}
        \State \(\mathcal{B}_\ell^{m + 1} \gets \mathcal{K}_\ell - \mathcal{A}_\ell^{m + 1}\). \Comment{Bad arm-set}
        \label{line:set_B}
        \State \(p_{k, \ell}^{m + 1} \gets \begin{cases}
            2^{-2(m + 1)}\frac{(\Delta_{k, \ell}^{m})^{-2}}{\sum_{k' \in \mathcal{K}_\ell}(\Delta_{k'}^{m})^{-2}}\frac{L_{\min}}{L_k}\frac{K_\ell}{K},                                                                              & \text{for arm }k \in \mathcal{B}_\ell^{m + 1}
            \\
            \frac{1}{|\mathcal{A}^{m + 1}_{\ell}|}\left(1 - \sum_{k' \in \mathcal{B}_\ell^{m + 1}}p_{k', \ell}^{m + 1}\right), & \text{for arm }k \in \mathcal{A}_\ell^{m + 1}
        \end{cases}\)
        \label{line:prob}
        \LeftComment{Phase 2: Arm Pulling and Communication}
        \For{\(t = T^{m - 1} + 1, T^{m - 1} + 2, \dots, T^{m}\)}
        \State Pick an arm \(k\) from \(\mathcal{K}_{\ell}\) according to probability \(p_{k, \ell}^m\) to pull.
        \label{line:arm_pull}
        \State Observe reward \(\tilde{r}_{k, \ell}^t\).
        \EndFor
        \State \(\tilde{R}_{k, \ell}^m \gets \sum_{t \in E^m}\tilde{r}_{k, \ell}^t\) for all arm \(k\in\mathcal{K}_\ell\).
        \label{line:sum_rewards}
        \State Broadcast \({\tilde{R}_{k, \ell}^m}\), \({p_{k, \ell}^m}\), and \(\Delta_{k, \ell}^{m - 1}\) for all arm \(k \in \mathcal{K}_{\ell}\) and the set \(\mathcal{A}_\ell^m\) to all agents \(\ell' \in \mathcal{L}\).
        \label{line:send_sum_rewards}
        \State Receive \({\tilde{R}_{k, {\ell'}}^m}\), \({p_{k, {\ell'}}^m}\), and \(\Delta_{k, \ell'}^{m - 1}\) for all arm \(k \in \mathcal K\) and the set \(\mathcal{A}_{\ell'}^m\) from all agents \(\ell' \in \mathcal{L}\).
        \label{line:receive_sum_rewards}
        % \State \(w_{k, \ell', \ell} \gets \frac{(\Delta_{k, \ell}^{m - 1})^{-2}}{p_{k, \ell}^m}\)\Comment{Weight Calculation}
        % \label{line:weights}
        \LeftComment{Phase 3: Reward Estimation}
        \State \(r_{k, \ell}^m \gets \text{Estimator}(\{\forall \ell' \in \mathcal{L}: p_{k, \ell'}^m\}, \{\forall \ell' \in \mathcal{L}: \tilde{R}_{k, \ell'}^m\}, T^m)\).\Comment{Reward Estimates}
        % \State \(r_{k, \ell}^m \gets \frac{\sum_{\ell' \in \mathcal{L}_k}{(p_{k, \ell'}^m)^{-1}}{\tilde{R}_{k, \ell}^m}}{T^m}\).\Comment{Reward Estimates}
        \label{line:approx_mean_reward}
        \State \(r_{\max, \ell}^m \gets \max_{k \in \mathcal K_{\ell}} r_{k, \ell}^m - \frac{1}{16}\Delta_{k, \ell}^{m - 1}\).
        \label{line:approx_max_reward}
        \State \(\Delta_{k, \ell}^{m} \gets \max{\{2^{-3}, r_{\max, \ell}^m - r_{k, \ell}^m + 3\times2^{-7}\}}\) for all arm \(k \in \mathcal{K}\).
        \label{line:epoch_gap}
        \EndFor
    \end{algorithmic}
\end{algorithm}

\paragraph{Algorithmic technical challenges}
The algorithmic design of \DRAA presents several technical challenges, especially due to the heterogeneous nature of the multi-agent system and the presence of adversarial corruption,
% . One key challenge arises from the differing arm-sets accessible to each agent, 
which complicates communication and the interpretation of shared information. Agents need to calculate an accurate estimation of their local arms. However, this estimation can easily be skewed due to the heterogeneous setting and adversarial corruption. Considering that a ``bad'' arm for one agent could be a ``good'' arm for another, estimation under this multi-agent heterogeneous setting requires careful information balancing, e.g., utilizing other agents' information with the right weights.

Another challenge lies in managing corruption while minimizing regret. Agents aim to minimize local regret, which involves pulling their best arms as frequently as possible based on reward estimates. However, adversarial corruption can make these estimates inaccurate. Therefore, traditional \MAB algorithms that eliminate ``bad'' arms to focus on exploiting good ones risk discarding the best arm if done under corruption. To maintain an optimal regret bound without corruption while adapting gracefully to increasing corruption, the algorithm must continue exploring all arms. This involves assigning careful pull probabilities to ``bad'' arms, sufficient to account for potential corruption but low enough to avoid a substantial regret increase.

% The algorithm introduces an arm-set construction technique that classifies arms into ``active'' or ``bad'' arm-sets, adjusting pull probabilities accordingly. 
% This technique limits the resources devoted to bad arms without entirely excluding them in case corruption impacts their reward estimates. 
% Specifically, the arm pull probabilities of bad arms are scaled based on the reward gap, the number of agents accessing each arm, and the number of arms accessed by the agent. This helps allocate resources to arms based on their estimated reward gaps, ensuring efficient resource distribution across all arms and agents, accounting for the setting’s heterogeneity.

% \xw{?}
% prevent overly frequent pulls of certain arms and balances the information flow across agents. This design effectively handles corruption while ensuring a close to optimal regret upper bound in the absence of corruption, and maintaining a low communication cost.

\paragraph{Active arm-set construction and exploration probability calculation (Line~\ref{line:T_m}-\ref{line:prob})}
\DRAA operates in epochs. The length of the \(m\)-th epoch is set to \(\lambda K 2^{2(m - 1)} / L_{\min}\) uniformly for all agents (Line~\ref{line:T_m}), ensuring synchronization in starting and finishing epochs. This allows agents to share their observations during the epoch and use collective information to update arm pull probabilities for the next one.
Each agent then classifies every arm in their local arm-set \(k \in \mathcal{K}_\ell\) as either active (\(\mathcal{A}_\ell^m\)) or bad (\(\mathcal{B}_\ell^m\)) based on the difference between the empirical mean of arm \(k\) and that of the local best arm from epoch \(m - 1\). Specifically, if \(r^{m - 1}_{\max, \ell} - r^{m - 1}_{k, \ell} \le 2^{-(m + 3)}\sqrt{{L_{\min}}/{L}} - 3 \cdot 2^{-7}\), arm \(k\) is placed in \(\mathcal{A}_\ell^m\) (Line~\ref{line:set_A}); otherwise, it is placed in \(\mathcal{B}_\ell^m\) (Line~\ref{line:set_B}).

Arm pull probabilities differ for active arms \(k \in \mathcal{A}_\ell^m\), and bad arms \(k \in \mathcal{B}_\ell^m\). The goal is to assign a small but non-zero probability to bad arms to maintain good regret performance while still collecting observations in case the reward estimates are corrupted (Line~\ref{line:prob}). For bad arms \(k \in \mathcal{B}_\ell^m\), the pull probability is proportional to the inverse of the empirical reward gap squared, \((\Delta_{k, \ell}^{m - 1})^{-2}\), normalized by the sum of such value for all arms in the agent's local arm-set. This ensures that arms with larger gaps are pulled less frequently. Using this normalization however, agents with smaller arm-sets are likely to pull an equally bad arm more frequently, so this probability is scaled by \(K_{\ell}/K\) to balance this effect. To prevent the number of bad arm pulls from growing too large as epochs length increases exponentially, the probability is further scaled by \(2^{-2m}\). It is also adjusted by \({{L_{\min}}/{L_{k}}}\), distributing the probability among all agents with access to the arm \(k\). The remaining probability is evenly distributed among active arms in \(\mathcal{A}_\ell^m\) (Line~\ref{line:prob}).

\paragraph{Arm-pulling and communication (Line~\ref{line:arm_pull}-\ref{line:receive_sum_rewards})}
In each round of epoch \(m\), agent \(\ell\) pulls arm \(k\) with probability \(p_{k, \ell}^m\) (Line~\ref{line:arm_pull}), and observes the reward. After the end of the epoch, the agent broadcasts to all other agents the sum of observed rewards \(\tilde{R}^m_{k, \ell}\), the arm pull probabilities \(p_{k, \ell}^m\), and the estimated local reward gaps of the previous epoch \(\Delta_{k, \ell}^{m - 1}\) for each arm \(k \in \mathcal{K}_\ell\), along with the set \(\mathcal{A}_\ell^m\) (Lines~\ref{line:send_sum_rewards} and \ref{line:receive_sum_rewards}).

\paragraph{Reward estimation (Line~\ref{line:approx_mean_reward}-\ref{line:epoch_gap})}
Minimizing total regret requires each agent to pull its local good arms as frequently as possible. To do this, agents must accurately estimate the empirical means of their local arms. However, empirical mean estimation is challenging due to (a) the heterogeneous multi-agent setting and (b) adversarial corruption. Different agents may have different sets of arms, leading to varying arm pull probabilities. The quality of an arm for an agent depends heavily on its local arm-set and is assessed relative to the agent’s local best arm. As a result, the estimated reward gap for the same arm \(k\) may differ between agents, causing them to assign different pull probabilities to the same arm. To address this challenge, we design a weighted estimator as follows
\begin{align}
\label{eq:weight_est}
r_{k, \ell}^m \coloneq \frac{\sum_{\ell' \in \mathcal{L}_k}{(p_{k, \ell'}^m)^{-1}}{\tilde{R}_{k, \ell}^m}}{L_kT^m}.
\end{align}
This estimator integrates observations from heterogeneous agents while minimizing bias by normalizing the sum of rewards observed by any agent \(\ell' \in \mathcal{L}_k\) using the inverse of that agent's original arm pull probability. We also test a naive estimator that simply averages all rewards for arm \(k\) across agents accessing it (\(\ell' \in \mathcal{L}_k\)), dividing by the expected value of the total number of pulls for arm \(k\) across all agents. The naive estimator is as follows
\begin{align}
\label{eq:naive_est}
r_{k, \ell}^m \coloneq \frac{\sum_{\ell' \in \mathcal{L}_k}{\tilde{R}_{k, \ell'}^m}}{\sum_{\ell' \in \mathcal{L}_k}{\tilde{p}_{k, \ell'}^m}T^m}.
\end{align}
We then compare both estimators theoretically Section~\ref{sec:4} and demonstrate that the weighted estimator achieves an optimal regret upper bound.
% We then compare both estimators theoretically and numerically in Sections~\ref{sec:4} and \ref{sec:5} and demonstrate that the weighted estimator achieves an optimal regret upper bound.

Then, the local maximum reward is computed as \(\max_{k \in \mathcal{K}_\ell} \left( r_{k, \ell}^m - \frac{1}{16} \Delta_{k, \ell}^{m - 1} \right)\) (Line~\ref{line:approx_max_reward}). Subtracting a fraction of the previous epoch's reward gap accounts for the arm's performance history and helps mitigate sudden changes in rewards due to corruption.
Each agent also calculates the empirical reward gap for the epoch as \(\max\{2^{-3}, r_{\max, \ell}^m - r_{k, \ell}^m\}\) for all arms \(k \in \mathcal{K}_\ell\) (Line~\ref{line:epoch_gap}). A minimum value is imposed on the reward gap to prevent any arm from being assigned an excessively high pull probability. 

% This ensures that even bad arms receive some pulls, helping to mitigate corruption while giving better arms the appropriate focus.

\paragraph{Comparison to the \barbar and \cite{liu2021cooperative}}
The \barbar algorithm operates in doubling epochs, estimating reward gaps at the end of each epoch and setting arm pull probabilities for the next epoch based on inverse-gap weighting. In contrast, our algorithm classifies arms into ``active'' and ``bad'' sets based on the difference between each arm's estimated mean and the highest estimated mean among local arms. This classification ensures that probabilities assigned to bad arms are inversely related to the total number of arms. The remaining probability is distributed among good arms, making their pulling probability inversely related to the number of good arms.
Our technique effectively improves a multiplicative \(K\) factor on the corruption \(C\) (even in the single-agent setting, detailed in Section~\ref{sec:4}), resolving an open problem mentioned in~\citep{gupta2019better}. 
Furthermore, in \barbar, epoch lengths are random variables determined by the arm-pulling probabilities. 
This randomness complicates extending the algorithm to distributed settings. 
However, when it comes to our heterogeneous arm-sets scenario, epoch lengths would vary across agents,
preventing synchronized epochs and the effective use of shared observations. To address this, we set each epoch length as a static value to ensure all agents are synchronized.

\citet{liu2021cooperative}'s algorithm does not account for the number of agents accessing each arm when assigning pull probabilities. This can lead to higher pull probabilities for bad arms accessible to multiple agents. In contrast, our algorithm divides the epoch length by \(L_{\min}\), the minimum number of agents with access to the same arm, and assigns pull probabilities for bad arms inversely proportional to the number of agents accessing those arms. 
This adjustment removes the factor of \(L\) from the corruption term in a homogeneous setting, addressing a key open problem in \cite{liu2021cooperative}.
For another thing, \cite{liu2021cooperative}'s algorithm follows a leader-follower structure, which differs from our fully distributed scheme. 
% also uses a doubling epoch approach with active and bad arm-sets. 
% However, their algorithm includes an arm allocation phase 
In their leader-follower setup, each agent only explores a subset of arms, and all of their observations are uploaded to a center leader.
The leader-follower framework in \cite{liu2021cooperative} aims to balance arm exploration among agents but poses challenges for extending the algorithm to heterogeneous settings. Additionally, in their algorithm, the leader agent calculates error levels globally based on information from all agents, which becomes impractical in heterogeneous contexts. In contrast, \DRAA is fully distributed, with reward gaps estimated locally. To integrate observations across agents, we use a carefully designed weighted estimator, enabling efficient use of communicated data in a local setting.

\section{Theoretical Results}\label{sec:4}

This section presents the theoretical regret upper bounds of \DRAA. We first present the most general regret bound in the heterogeneous setting for \DRAA using our weighted estimator (Eq.~\eqref{eq:weight_est}), and the naive estimator (Eq.~\eqref{eq:naive_est}) to demonstrate the effect of our carefully devised estimator on the performance of the algorithm. 
Then, we reduce the bound to two popular special cases to illustrate the tightness of our results. 
Specifically, in Theorem~\ref{theo:1}, we show that the regret upper bound of \DRAA using the weighted estimator introduced in Eq.~\ref{eq:weight_est} is optimal in the fully stochastic, no-corruption case, up to a \(\log(\log T)\) factor, with an additive term of \(O\left(({L}/{L_{\min}})C\right)\), where \(L_{\min}\) is the minimum number of agents with mutual access to the same arm. 
In Remark~\ref{remark:1}, we also demonstrate that using the weighted estimator can improve the corruption term by a \(L_{\min}\) factor.
Additionally, Corollaries~\ref{cor:1} and \ref{cor:2} demonstrate that in the special cases of single-agent \MAB and homogeneous \MAtB, our regret upper bound matches the problem’s lower bound and improves upon previous upper bounds proposed in \cite{gupta2019better} and \cite{liu2021cooperative}.

% In the heterogeneous multi-agent setting, each agent has access to a potentially different subset of arms, which may overlap with other agents' arm sets. This variation complicates communication and the interpretation of shared information. For instance, an agent with access to fewer arms will pull an arm with a higher probability than an agent with access to more arms, as resources are spread across all available arms. Another challenge is balancing the allocation of resources to bad arms to mitigate corruption without significantly increasing stochastic regret in the absence of corruption.

% To address these challenges, we implemented new algorithmic techniques, such as splitting arms into active and bad sets, assigning different probabilities to each set, and using weighted averaging schemes to calculate the empirical mean of arms in both sets. 

In Theorem~\ref{theo:1}, we state the high probability regret upper bound for \DRAA using the weighted estimator presented in Eq.~\eqref{eq:weight_est}.
\begin{theorem}[\DRAA regret upper bound]
    \label{theo:1}\label{thm:main-result}
    With a probability \(1 - \delta\), the \emph{\DRAA} algorithm (Algorithm~\ref{alg:alg}) using our weighted estimator incurs \(O\left(L\log\left(\frac{T}{\log(({8K}/{\delta})\log T)}\right)\right)\) communication cost, and its regret is upper bounded by
    \[
        O\left(\frac{L}{L_{\min}}C  + \log{\left(\frac{KL}{\delta}\log{T}\right)}\log{T}\frac{K}{\Delta_{\min}}\right),
    \]
    where \(L_{\min}\coloneqq \min_{k\in\mathcal K} L_k \) is the minimum number of agents with mutual access to any of the arms, and \(\Delta_{\min}\) is the minimum non-zero global arm-gap.
\end{theorem}
Theorem~\ref{theo:1} demonstrates that the impact of corruption on regret decreases as the agents' access to arms becomes more uniform. The complete proof of Theorem~\ref{theo:1} can be found in Appendix~\ref{appendix:1}. In Remark~\ref{remark:1} we state a high probability regret upper bound for \DRAA using the naive estimator presented in Eq.~\eqref{eq:naive_est}.
\begin{remark}
    \label{remark:1}
    With a probability \(1 - \delta\), the \emph{\DRAA} algorithm (Algorithm~\ref{alg:alg}) using a naive estimator incurs a regret upper bounded by
    \[
        O\left(LC  + \log{\left(\frac{KL}{\delta}\log{T}\right)}\log{T}\frac{K}{\Delta_{\min}}\right).
    \]
\end{remark}
Comparing Theorem~\ref{theo:1} with Remark~\ref{remark:1} reveals that the weighted estimator outperforms the naive estimator by effectively leveraging the overlap in agents' arm sets, thereby minimizing the regret increase from adversarial corruption.
This also highlights the importance of devising the sophisticated estimator in~\eqref{eq:weight_est}. The complete proof of Remark~\ref{remark:1} is presented in Appendix~\ref{appendix:2}.

\paragraph{Technical challenges} The new techniques in Algorithm~\ref{alg:alg} introduce specific challenges in analyzing its regret upper bound. \begin{itemize}
    \item \textbf{Bounding arm-pull probabilities for arms in active/bad sets:} The algorithm divides each agent’s local arms into active and bad sets, assigning different pull probabilities based on these sets. As a result, the regret depends on both an arm’s local reward gap and its set assignment. To address this, we first bound the pull probability for agent \(\ell\) for each arm in its active and bad sets in epoch \(m\). Additionally, to bound the regret incurred by arm \(k\) for agent \(\ell\) in rounds during epoch \(m\), we examine whether the arm is in the agent's active or bad arm set and establish the upper bounds for each scenario accordingly.
    \item \textbf{Bounding the weighted estimator:} As specified in Algorithm~\ref{alg:alg}, we estimate this probability using a weighted average, where each observed reward of an arm is weighted by the inverse of the arm’s pull probability, \((p_{k, \ell}^m)^{-1}\). This bound differs depending on whether arm \(k\) is in the active or bad agent set \(\ell\). To bound the difference between the estimated mean reward and its actual value, we use two concentration inequalities that depend on the upper bound of the random variable. This approach requires separate bounds for agents \(\ell'\) that have \(k\) as an active arm and those that assign it as a bad arm. By introducing a multiplicative factor in both the numerator and denominator, we bound the random variable by 1 in both cases, simplifying the proof of Lemma~\ref{lem:2}. 
\item \textbf{Setting reward-gap threshold for the cases:} When analyzing cases based on the reward gap and corruption level for arm \(k\) and agent \(\ell\) in epoch \(m\), we determine whether the arm should belong to the agent's active or bad set. We set a small threshold on the reward gap, ensuring that when the reward gap is large, and corruption is small, arm \(k\) is placed in the agent’s bad arm-set. This threshold acts as a critical upper bound for the reward gap—if set too high or too low,if set too high or too low, the regret upper bound will increase correspondingly for either the case with a low reward gap or the case with a high reward gap and low corruption. With this carefully chosen threshold, we balance the regret upper bounds across cases, achieving the tightest possible bound.
\end{itemize}

\paragraph{Proof Sketch} 
Here we provide a proof sketch for Theorem~\ref{theo:1}. 
In Lemma~\ref{lem:1}, we separately bound the arm-pull probabilities of each agent for arms in \(\mathcal{A}_{\ell}^m\) (active arms) and \(\mathcal{B}_{\ell}^m\) (bad arms). Then, in Lemma~\ref{lem:2}, we show that the empirical mean reward of arm \(k\) estimated by agent \(\ell\) for epoch \(m\) closely approximates the actual mean reward \(\mu_k\), with their difference bounded by \(\frac{2C^m}{L_{\min}T^m} + \frac{\Delta_{k, \ell}^{m-1}}{16}\). This bound is intuitively governed by two factors: the estimated reward gap from the previous epoch and the maximum amount of observed corruption for any agent over one round in epoch \(m\). 
In Lemma~\ref{lem:2}, we also show that the actual number of pulls of arm \(k\) by agent \(\ell\) during epoch \(m\) is less than two times its expected value, \(p_{k, \ell}^m T^m\) with high probability. 
Following this, we bound the estimated reward gap for each arm \(k\) for agent \(\ell\) in epoch \(m\), denoting the upper bound as \(2(\Delta_{k, \ell} + \sqrt{\frac{L_{\min}}{L}}2^{-m} + \rho^m + 2^{-4})\) in Lemma~\ref{lem:4} and the lower bound as \(\frac{1}{2} \Delta_{k, \ell} - 3\rho^m -\frac{3}{4}\sqrt{\frac{L_{\min}}{L}}2^{-m}\) in Lemma~\ref{lem:5}. 
Here, \(\rho_m\) represents the cumulative corruption per agent up to epoch \(m\). Intuitively, a smaller \(\rho_m\) will result in a closer estimates of the reward gaps. Furthermore, in later epochs, as agents gather more information about the arms, these reward gap estimates become increasingly accurate.

Next, we upper-bound the regret. Using Lemma~\ref{lem:2}, the total regret can be bounded as \(\Reg_T = 2\sum_{m=1}^{M} \sum_{\ell \in \mathcal{L}} \sum_{k \in \mathcal{K}_\ell} \Delta_{k, \ell} p^m_{k, \ell} T^m\) with high probability, replacing the actual number of arm pulls by their expected value. We then bound the regret for each arm \(k\) for each agent \(\ell\) in epoch \(m\) --- \(\Delta_{k, \ell} p^m_{k, \ell} T^m\) --- by analyzing the following three cases

\textbf{Case 1:} The local reward gap for arm \(k\) is smaller than a set threshold, \(0 \le \Delta_{k, \ell} \le \frac{4}{2^{m}} \sqrt{\frac{L_{\min}}{L}}\). We apply Lemma~\ref{lem:1} to bound the regret for both situations where arm \(k\) is in the active or bad arm set of agent \(\ell\) in epoch \(m\). In this case, the regret is bounded by \(O(\log{\left(\frac{KL}{\delta}\log{T}\right)}\log{T}\frac{K}{\Delta_{\min}})\).

\textbf{Case 2:} The local reward gap for arm \(k\) is larger than the threshold, \(\Delta_{k, \ell} \ge \frac{4}{2^{m}} \sqrt{\frac{L_{\min}}{L}}\), and corruption up to epoch \(m\) is small, \(\rho^{m - 1} < \Delta_{k, \ell}/32\). Here, the observed reward for the arm remains close to its original stochastic value, and we set the gap threshold so that arm \(k\) consistently falls into the bad arm set of agent \(\ell\). As the arm-pull probability of arms in the bad arm-set depends on the inverse of the estimated reward gap of the arm, we need to use the lower bound for this value (Lemma~\ref{lem:5}) to bound the regret of this case, again resulting in an upper bound of \(O(\log{\left(\frac{KL}{\delta}\log{T}\right)}\log{T}\frac{K}{\Delta_{\min}})\). Achieving the same upper bound as Case 1 further confirms that our analysis is tight. 

\textbf{Case 3:} The local reward gap for arm \(k\) is larger than the threshold, \(\Delta_{k, \ell} \ge \frac{4}{2^{m}} \sqrt{\frac{L_{\min}}{L}}\), and corruption up to epoch \(m\) is large, \(\rho^{m - 1} > \Delta_{k, \ell}/32\). Here, the observed reward may fluctuate significantly, and regret is mainly controlled by corruption. We bound the effect of corruption in both scenarios, where arm \(k\) is in either the active or bad arm set of agent \(\ell\) in epoch \(m\). In this case, the regret is upper bounded by \(O(({L}/{L_{\min}})C)\).

This case analysis, with specific bounds for each scenario, ensures a comprehensive evaluation of the regret upper bound across varying reward gaps and corruption levels.

\subsection{Special cases for single-Agent \MAB and homogeneous \MAtB}

In this subsection, we specialize Theorem~\ref{theo:1} to the single agent and homogeneous multi-agent scenarios. 
For the case of a single-agent \MAB with adversarial corruption,
our regret upper bound in Theorem~\ref{theo:1} can be reduced to Corollary~\ref{cor:1}.
\begin{corollary}
\label{cor:1}
    In a single-agent case, where \(L = L_{\min} = 1\), with probability at least \(1 - \delta\), the regret of Algorithm~\ref{alg:alg} is bounded as
    \[
    O\left(C +  \log{\left(\frac{K}{\delta}\log{T}\right)}\log{T}\frac{K}{\Delta_{\min}}\right).
    \]
\end{corollary}
Corollary~\ref{cor:1} establishes that the regret upper bound of Algorithm~\ref{alg:alg} matches the lower bound proved by \cite{gupta2019better}, improving on \barbar's upper bound by removing a multiplicative \(K\) from the corruption term and thus resolving the primary open question mentioned in their Discussions section.

Furthermore, Corollary~\ref{cor:2} shows that in the case of a homogeneous \MAtB with adversarial corruption, our Theorem~\ref{theo:1} also implies a tight regret bound.
\begin{corollary}
\label{cor:2}
    In a homogeneous multi-agent scenario with \(K\) arms and \(L\) agents, we have \(L_{\min} = L\). Hence, with probability at least \(1 - \delta\), the regret of Algorithm~\ref{alg:alg} is bounded as
    \[
    O\left(C +  \log{\left(\frac{KL}{\delta}\log{T}\right)}\log{T}\frac{K}{\Delta_{\min}}\right)
    \]
\end{corollary}
Corollary~\ref{cor:2} shows that the regret upper bound of Algorithm~\ref{alg:alg} in a homogeneous MAB setting matches the lower bound proved by \cite{liu2021cooperative}, removing a multiplicative \(L\) from the corruption term and resolving their main open problem.

\paragraph{Regret Tightness} 
The regret upper bound of Algorithm~\ref{alg:alg} aligns with the lower bound of stochastic MAB up to a logarithmic factor. While our regret bound is \(\tilde{O}(\log T \, K / \Delta_{\min})\), which may appear weaker than \(O(\log T \sum_{k \neq k^*} 1/\Delta_k)\), the leading term in both cases remains \(1 / \Delta_{\min}\). With corruption level \(C\), Algorithm~\ref{alg:alg} matches the lower bounds for both single-agent MAB and homogeneous MAB, as shown in \cite{gupta2019better} and \cite{liu2021cooperative}, respectively, adding only a linear \(C\) factor to the regret upper bound.

\section{Conclusion}\label{sec:6}

In this paper, we introduce the algorithm \DRAA, designed for the \texttt{MA2B} setting with heterogeneous agents, ensuring robustness to adversarial corruption. \DRAA employs a fully distributed scenario, where all agents possess equal roles and execute the same algorithm. This algorithm take logarithmic communication costs to achieve a regret upper bound that includes an additive term linearly related to the corruption level $C$. This term is in addition to the standard regret in non-corruption settings. Our theoretical results also show that \DRAA outperforms existing state-of-the-art single-agent \MAB and homogeneous \MAtB algorithms, achieving theoretical lower bounds in these settings as an additional outcome of our approach.

This work also opens up multiple future directions. 
We mainly focus on adversarial reward corruption in this paper, while exploring communication corruption presents an interesting direction. In such scenarios, the adversary could corrupt the shared information among agents, adding a complex layer to the problem. 
Furthermore, although our algorithm aligns with previously established lower bounds in two specific cases, proving lower bounds for the heterogeneous \MAtB setting remains an open problem.

% In this paper, we introduce the algorithm \DRAA, designed for the \texttt{MA2B} setting with heterogeneous agents, ensuring robustness to adversarial corruption. \lfbarbar employs a leader-follower scheme, where only the leader is responsible for arm selection, guiding the actions of all follower agents. In contrast, \MultiDRAA operates in a fully distributed scenario, where all agents possess equal roles and execute the same algorithm. Both algorithms take logarithmic communication costs to achieve a regret upper bound that includes an additive term linearly related to the corruption level $C$. This term is in addition to the standard regret in non-corruption settings, and we also prove that our upper bound analysis is tight.
% Numerical results demonstrate the superior performance of our algorithms over two baselines of non-robust and non-cooperative algorithms across various settings. 
% % \(O\left(\frac{KC}{L_{\min}}\right) + \Tilde O(\sum_{k: \Tilde{\Delta}_k<0}\log{T}/\Tilde{\Delta}_k)\)

% This work also opens up multiple future directions. 
% We mainly focus on adversarial reward corruption in this paper, while exploring communication corruption presents an interesting direction. In such scenarios, the adversary could corrupt the shared information among agents, adding a complex layer to the problem.
% Additionally, it is an open problem about whether we can eliminate the dependence on \(K\) from the corruption term.
% % while we have established that the dependence of the corruption term on $C$ is tight in our analysis,

\clearpage
\bibliography{iclr2025_conference}
\bibliographystyle{iclr2025_conference}

\newpage
\appendix
\section{Proof for Theorem~\ref{thm:main-result}}
\label{appendix:1}
\subsection{Key Lemmas}
\begin{lemma}\label{lem:1}
    It holds that for every agent \(\ell \in \mathcal{L}\) and epoch \(m\), for all \(k \in \mathcal{B}^m_\ell\)
    \begin{align}
    \label{eq:prob_B_bounds}
        \frac{L_{\min}}{L_k}\frac{2^{-2m - 7}}{K} \le p^m_{k, \ell} \le \frac{L_{\min}}{L_k}\frac{2^{-2m + 7}}{K},
    \end{align}
    and for all \(k \in \mathcal{A}^m_\ell\)
    \begin{align}
    \label{eq:prob_A_bounds}
        \frac{3}{4|\mathcal{A}_\ell^m|} \le p^m_{k, \ell} \le \frac{1}{|\mathcal{A}_\ell^m|}. 
    \end{align}
\end{lemma}

\begin{proof}
    \item 
    \paragraph{Bounding \(\pmb{p_{k, \ell}^m}\) for arms \(\pmb{k \in \mathcal{B}^m_\ell}\)} We know that for every agent \(\ell \in \mathcal{L}\) and epoch \(m\), the probability of pulling any arm \(k \in \mathcal{B}^m_\ell\) is equal to    
    \[
        p_{k, \ell}^m = 2^{-2m}\frac{(\Delta_{k, \ell}^{m - 1})^{-2}}{\sum_{k' \in \mathcal{K}_\ell}(\Delta_{k', \ell}^{m - 1})^{-2}}\frac{L_{\min}}{L_k}\frac{K_\ell}{K}.
    \]
    To bound this probability, we use the face that for any agent \(\ell\), epoch \(m\), and arm \(k\), the empirical reward gap is bounded as \(2^{-3} \le \Delta_{k, \ell}^m \le 1 + 3\times2^{-7}\le1.03\), due to the definition of this random variable at Line~\ref{line:epoch_gap} in Algorithm~\ref{alg:alg}. Hence, the lower bound of this probability would be as follows
    \begin{align*}
        p_{k, \ell}^m = 2^{-2m}\frac{(\Delta_{k, \ell}^{m - 1})^{-2}}{\sum_{k' \in \mathcal{K}_\ell}(\Delta_{k', ell}^{m})^{-2}}\frac{L_{\min}}{L_k}\frac{K_\ell}{K} \ge 2^{-2m}\frac{0.94}{K_\ell \cdot 2^6}\frac{L_{\min}}{L_k}\frac{K_\ell}{K} \ge 2^{-2m - 7}\frac{L_{\min}}{K L_k}.
    \end{align*}
    We can further prove the upper bound of Equation~\eqref{eq:prob_B_bounds} as
    \begin{align*}
         p_{k, \ell}^m = 2^{-2m}\frac{(\Delta_{k, \ell}^{m - 1})^{-2}}{\sum_{k' \in \mathcal{K}_\ell}(\Delta_{k', \ell}^{m - 1})^{-2}}\frac{L_{\min}}{L_k}\frac{K_\ell}{K} \le 2^{-2m}\frac{2^6}{0.94K_\ell}\frac{L_{\min}}{L_k}\frac{K_\ell}{K} \le 2^{-2m + 7}\frac{L_{\min}}{K L_k}
    \end{align*}
    
    \item 
    \paragraph{Bounding \(\pmb{p_{k, \ell}^m}\) for arms \(\pmb{k \in \mathcal{A}^m_\ell}\)}
        We know that for every agent \(\ell \in \mathcal{L}\) and epoch \(m\), the probability of pulling any arm \(k \in \mathcal{A}^m_\ell\) is equal to
        \[
        p_{k, \ell}^m = \frac{1}{|\mathcal{A}^m_{\ell}|}\left(1 - \sum_{k' \in \mathcal{B}^m_\ell}p^m_{k', \ell}\right).
        \]
        To upper bound this probability, we can use the fact the sum of the arm pull probabilities of all the arms in \(\mathcal{B}^m_{\ell}\) is at least \(0\).
        \begin{align*}
            p_{k, \ell}^m = \frac{1}{|\mathcal{A}^m_{\ell}|}\left(1 - \sum_{k' \in \mathcal{B}^m_\ell}p^m_{k', \ell}\right) \le  \frac{1}{|\mathcal{A}^m_{\ell}|}.
        \end{align*}
        To prove the lower bound of Equation~\eqref{eq:prob_A_bounds}, we an write
        \begin{align*}
              p_{k, \ell}^m =  \frac{1}{|\mathcal{A}^m_{\ell}|}\left(1 - \sum_{k' \in \mathcal{B}^m_\ell}p^m_{k', \ell}\right) =& \frac{1}{|\mathcal{A}^m_{\ell}|}\left(1 - 2^{-2m}\frac{\sum_{k' \in B^m_\ell}(\Delta_{k'}^{m - 1})^{-2}}{\sum_{k' \in \mathcal{K}_\ell}(\Delta_{k'}^{m - 1})^{-2}}\frac{L_{\min}}{L_k}\frac{K_\ell}{K}\right)\\
              \overset{(a)}{\ge}& \frac{1}{|\mathcal{A}^m_{\ell}|}\left(1 - 2^{-2}\cdot 1 \cdot 1 \cdot 1\right) = \frac{3}{4|\mathcal{A}^m_{\ell}|},
        \end{align*}
        where (a) follows the fact that \(2^{-2m}\) is upper bounded by \(2^{-2}\) as \(m \ge 1\), and \(\frac{L_{\min}}{L_k}\), \(\frac{K_\ell}{K}\), and \(\frac{\sum_{k' \in B^m_\ell}(\Delta_{k'}^{m - 1})^{-2}}{\sum_{k' \in \mathcal{K}_\ell}(\Delta_{k'}^{m - 1})^{-2}}\) are upper bounded by 1.
 \end{proof}

 We define an event \(\mathcal{E}\) such as:
\begin{align}\label{eq:event-E_MA}
    \mathcal{E} \coloneqq \left\{\forall \ell \in \mathcal{L}, \forall k \in \mathcal{K}_\ell, m \in \{1, \dots, M\} : \left|r_{k, \ell}^m - \mu_k\right| \le \frac{2C^m}{L_{\min}T^m} + \frac{\Delta_{k, \ell}^{m-1}}{16} \text{ and } \Tilde{n}_{k, \ell}^m \le {2 p^m_{k, \ell}T^m} \right\},
\end{align}
where \(C^m\) is the total corruption in epoch \(m\), and \(\Tilde{n}_{k, \ell}^m\) is the actual number of pulls of arm \(k\) in epoch \(m\) by agent \(\ell\).

\begin{lemma}\label{lem:2}
    It holds that:
    \begin{align}
        \Pr[\mathcal{E}] \geq 1 - \delta.
    \end{align}
\end{lemma}

\begin{proof}
    % \(\beta \geq 4e^{-\frac{\lambda}{16L}}\), where \(L\) is the total number of agents.
    % For any fixed \(\ell \in \mathcal{L}\), \(k \in \mathcal{K}_{\ell}\), and \(m\) we have:
    \item
    \paragraph{Decomposing \(\pmb{r_{k, \ell}^m}\)} For this proof, we condition on all random variables before some fixed epoch \(m\), so that \(p^m_{k, \ell}\) is a deterministic quantity. At each step in this epoch, agent \(\ell\) picks arm \(k \in \mathcal{K}_\ell\) with probability \(p_{k, \ell}^m\). Let \(Y_{k, \ell}^t\) be an indicator for arm \(k\) being pulled in step \(t\) by agent \(\ell\). Let \(R_{k, \ell}^t\) be the stochastic reward of arm \(k\) on step \(t\) for agent \(\ell\). and \(c^t_{k, \ell} \coloneqq \tilde{R}_{k, \ell}^t - R_{k, \ell}^t\) be the corruption added to arm \(k\) by the adversary in step \(t\) for agent \(\ell\). Note that \(c^t_{k, \ell}\) may depend on all the stochastic rewards up to (and including) round \(t\), and also on all previous choices of the agents (though not the choice at step \(t\)). Finally, we denote \(E^m\) to be the \(T^m\) timesteps in epoch \(m\), then we can write the random variable \(r_{k, \ell}^m\) defined in Line~\ref{line:approx_mean_reward} of Algorithm~\ref{alg:alg} as follows
    \begin{equation}
        r_{k, \ell}^m = \frac{\sum_{\ell' \in \mathcal{L}_k}{(p_{k, \ell'}^m)^{-1}}\sum_{t \in E^m}Y_{k, \ell'}^t\left(R_{k, \ell'}^t + c_{k, \ell'}^t\right)}{L_k T^m}.
    \end{equation}
    Let us define two auxiliary random variables \(A_{k, \ell}^m\) and \(B_{k, \ell}^m\) as
    \begin{align*}
        A_{k, \ell}^m & = {w_{k, \ell}^m}{\sum_{\ell' \in \mathcal{L}_k}(p_{k, \ell'}^m)^{-1}\sum_{t \in E^m}Y_{k, \ell'}^t R_{k, \ell'}^t}\\
        B_{k, \ell}^m & = {w_{k, \ell}^m}{\sum_{\ell' \in \mathcal{L}_k}(p_{k, \ell'}^m)^{-1}\sum_{t \in E^m}Y_{k, \ell'}^t c_{k, \ell'}^t},
    \end{align*}
    where we define \(w_{k, \ell}^m\) as
    \[
    {w_{k, \ell}^m} = \frac{2^{-2m - 7}}{K} \cdot \frac{L_{\min}}{L_k} \cdot (\Delta_{k, \ell}^{m + 1})^{-2}.
    \]
    \item
    \paragraph{Bounding \(\pmb{A_{k, \ell}^m}\)} To bound \(A_{k, \ell}^m\), we first calculate its expected value as
    \begin{align*}
        \mathbb{E}\left[A_{k, \ell}^m\right] = {w_{k, \ell}^m}{\sum_{\ell' \in \mathcal{L}_k}(p_{k, \ell'}^m)^{-1}\sum_{t \in E^m}\mathbb{E}\left[Y_{k, \ell'}^t R_{k, \ell'}^t\right]} = {w_{k, \ell}^m}\sum_{\ell' \in \mathcal{L}_k}(p_{k, \ell'}^m)^{-1}T^m\left(p_{k, \ell'}^m\mu_k\right) = w_{k, \ell}^m L_k T^m \mu_k,
    \end{align*}
    Therefore, according to the Chernoff-Hoeffding bound \cite{dubhashi2009concentration}, and because \(\mu_k \le 1\), for \(\beta \le 1\) we have:
    \begin{align}
    \label{eq:A_ineq}
        Pr\left[\left|\frac{A_{k, \ell}^m}{w_{k, \ell}^m L_k T^m} - \mu_k\right| \ge \sqrt{\frac{3\ln{\frac{4}{\beta}}}{w_{k, \ell}^m L_k T^m}}\right] \le \frac{\beta}{2}.
    \end{align}

     \item
    \paragraph{Bounding \(\pmb{B_{k, \ell}^m}\)}
    Next, we aim to bound \(B_{k, \ell}^m\). To this end, we define the sequence of r.v.s, \(\{X_{t'}\}_{t'\in[1,\dots, L_kT]}\), where \(X_{t'} = \left(Y_{k, \ell}^t - p_{k, \ell}^m\right)\cdot {w_{k, \ell}^m} (p_{k, \ell'}^m)^{-1} c_{k, \ell}^t\) for all \(t'\), where \(t = \lceil  t'/L_k \rceil\), and \(\ell = t' \text{ mod } t\). Then \(\{X_{t'}\}_{t'\in[1,\dots, L_kT]}\) is a martingale sequence with respect to the filtration \(\{\mathcal{F}_{t'}\}_{t'=1}^{L_kT}\) generated by r.v.s \(\{Y_{k, h}^s\}_{k \in \mathcal{K}, h \in \mathcal{L}_k, s \le t}\) and \(\{R_{k, h}^s\}_{k \in \mathcal{K}, h \in \mathcal{L}_k, s \le t}\). The corruption \(c_{k, \ell}^m\) becomes deterministic conditioned on \(\mathcal{F}_{t-1}\), and since \(\mathbb{E}\left[Y_{k, \ell}^t | \mathcal{F}_{t-1}\right] = p_{k, \ell}^m\), we have:
    \begin{align*}
        \mathbb{E}\left[X_{t'} | \mathcal{F}_{t-1}\right] = \sum_{\ell' \in \mathcal{L}}\sum_{t \in E^m}\mathbb{E}\left[Y_{k, \ell'}^t - p_{k, \ell'}^m | \mathcal{F}_{t-1}\right]\cdot 2^{-6} {w_{k, \ell', \ell}^m} c_{k, \ell'}^t = 0.
    \end{align*}
    Using a Freedman-type concentration inequality introduced in \citeauthor{beygelzimer2011contextual}, the predictive quadratic variation of this martingale can be bounded as:
    \begin{gather*}
        \Pr\left[\sum_{t'} X_{t} \geq \frac{V}{b} + b\ln \frac{4}{\beta}\right] \le \frac{\beta}{4},
    \end{gather*}
    where \(\left|X_{t'}\right| \le b\), and
    \begin{align*}
        V = & \sum_{\ell' \in \mathcal{L}}\sum_{t \in E^m}\mathbb{E}\left[X_{t'}^2 | \mathcal{F}_{t-1}\right] \le \sum_{\ell' \in \mathcal{L}}{w_{k, \ell}^m} (p_{k, \ell'}^m)^{-1}\sum_{t \in E^m} \left|c_{k, \ell'}^t\right|\text{Var}\left(Y_{k, \ell'}^t\right) \le {w_{k, \ell}^m} \sum_{\ell' \in \mathcal{L}} \sum_{t \in E^m} \left|c_{k, \ell'}^t\right|.
    \end{align*}
    Upper bounding \(\left|X_{t'}\right|\), we have
    \begin{align*}
    \left|X_{t'}\right| = \left(Y_{k, \ell}^t - p_{k, \ell}^m\right)\cdot\frac{2^{-2m - 7}}{K} \cdot \frac{L_{\min}}{L_k} \cdot (\Delta_{k, \ell}^{m + 1})^{-2}\cdot(p_{k, \ell'}^m)^{-1} c_{k, \ell}^t
    \end{align*}
    Therefore if \(k \in \mathcal{A}^m_{\ell'}\) we have
    \begin{align*}
    \left|X_{t'}\right| \le \frac{2^{-2m - 13}}{K} \cdot \frac{L_{\min}}{L_k} \cdot (\Delta_{k, \ell}^{m + 1})^{-2}\cdot\frac{4}{3}\left|\mathcal{A}_{\ell'}^m\right| \le \frac{L_{\min}}{L_k} \cdot \frac{\left|\mathcal{A}_{\ell'}^m\right|}{K} \cdot 2^{-2m - 13} \cdot \frac{4}{3} \cdot 2^{6} \le 1, 
    \end{align*}
    and if \(k \in \mathcal{B}^m_{\ell'}\) we have
    \begin{align*}
    \left|X_{t'}\right| \le \frac{2^{-2m - 13}}{K} \cdot \frac{L_{\min}}{L_k} \cdot (\Delta_{k, \ell}^{m + 1})^{-2}\cdot \frac{L_k}{L_{\min}}\frac{K}{2^{-2m - 7}}\le  2^{-6}(\Delta_{k, \ell}^{m + 1})^{-2} \le 1.
    \end{align*}
    Therefore, we obtain that with a probability of \(\frac{\beta}{4}\):
    \begin{align}
        \frac{B_{k, \ell}^m}{{w_{k, \ell}^m}L_kT^m} & \overset{(a)}\ge \frac{{w_{k, \ell}^m}\sum_{\ell' \in \mathcal{L}} \sum_{t \in E^m} c_{k, \ell'}^t}{{w_{k, \ell}^m}L_kT^m} + \frac{\left(V + \ln{(4/\beta)}\right)}{{w_{k, \ell}^m}L_kT^m} \nonumber\\
        & \ge 2\frac{{w_{k, \ell}^m}\sum_{\ell' \in \mathcal{L}} \sum_{t \in E^m} \left|c_{k, \ell'}^t\right|}{{w_{k, \ell}^m}L_kT^m} + \frac{\ln{(4/\beta)}}{{w_{k, \ell}^m}L_kT^m},\label{eq:B_ineq_1}
    \end{align}

    where (a) is because \(\sum_{t'}X_{t'} = \sum_{\ell' \in \mathcal{L}_k}\sum_{t \in E^m}(Y_{k, \ell'}^t - p_{k, \ell'}^m)\cdot {w_{k, \ell}^m}(p_{k, \ell'}^m)^{-1} c_{k, \ell'}^t = B_{k}^m - {w_{k, \ell}^m}\sum_{\ell' \in \mathcal{L}}\sum_{t \in E^m} c_{k, \ell'}^t\).

    We also define the total corruption on agent \(\ell\)'s rewards in epoch \(m\) as \(C^m_\ell = \sum_{t \in E^m}\max_{k \in \mathcal{K}} |c_{k, \ell}^t|\), therefore we can conclude that \(\sum_{\ell \in \mathcal{L}} \left|c_{k, \ell}^t\right| \le C^m\). We choose a parameter \(\beta \geq 4e^{(-\lambda/16)}\), therefore
    \begin{align}
        {w_{k, \ell}^m}L_kT^m \overset{(a)}{\ge} & \frac{2^{-2m -13}}{K} \cdot \frac{L_{\min}}{L_k} \cdot (\Delta_{k, \ell}^{m - 1})^{-2}\cdot L_k \cdot T^m\nonumber\\
        \ge & 2^{-13} \left(\frac{2^{-2m}L_{\min}(\Delta_{k, \ell}^{m - 1})^{-2}}{K}\right)T^m\label{eq:lower_boumd_denom_1}\\
        \ge& 2^{-13} \left(\frac{2^{-2m}L_{\min}(\Delta_{k, \ell}^m)^{-2}}{K} \cdot \frac{2^{2(m - 1)}\lambda K}{L_{\min}}\right)\nonumber\\
        \ge & 2^{10}\ln(\frac{4}{\beta})
        \ge 16\ln{(\frac{4}{\beta})},
    \end{align}
    where in step (a) we are splitting the set of all agents \(\ell'\) in \(\mathcal{L}_k\) into the set of agents that have \(k\) in \(|\mathcal{A}_{\ell'}^m|\) and the ones that have \(k\) in \(|\mathcal{B}_{\ell'}^m|\). Afterwards, we use Lemma~\ref{lem:1} to lower bound the probability \(p^m_{k, \ell'}\) for each set of agents. 
    Now we can rewrite Equation~\eqref{eq:B_ineq_1} as:
   \begin{align}
        \label{eq:B_ineq_2}
        \frac{B_{k, \ell}^m}{{w_{k, \ell}^m}L_kT^m}\ge 2\frac{{w_{k, \ell}^m}\sum_{\ell' \in \mathcal{L}} \sum_{t \in E^m} c_{k, \ell'}^t}{{w_{k, \ell}^m}L_kT^m} + \sqrt{\frac{\ln{4/\beta}}{16\cdot{w_{k, \ell}^m}L_kT^m}}.
    \end{align}
    In the next step, we aim to bound the corruption term of the upper bound stated in Equation~\eqref{eq:B_ineq_2}. To this end, we first bound the nominator of the corruption term as follows
    \begin{align}
        {w_{k, \ell}^m}\sum_{\ell' \in \mathcal{L}} \sum_{t \in E^m} c_{k, \ell'}^t \overset{(a)}{\le} & \sum_{{\ell' \in \mathcal{L}_k}}\left(\frac{2^{-2m - 13}}{K} \cdot \frac{L_{\min}}{L_k} \cdot (\Delta_{k, \ell}^{m - 1})^{-2}\right)C^m_{\ell'}\nonumber\\
        \le & \frac{2^{-2m - 13}}{K} \cdot \frac{L_{\min}}{L_k} \cdot (\Delta_{k, \ell}^{m - 1})^{-2}\cdot C^m.\label{eq:corr_upper_bound_nom}
    \end{align}
    Afterwards, we can write
    \begin{align*}
        \frac{2{w_{k, \ell}^m}\sum_{\ell' \in \mathcal{L}} \sum_{t \in E^m} c_{k, \ell'}^t}{{w_{k, \ell}^m}L_kT^m} \overset{(a)}{\le}&\frac{\frac{2^{-2m - 12}}{K} \cdot \frac{L_{\min}}{L_k} \cdot (\Delta_{k, \ell}^{m + 1})^{-2}C^m}{\frac{2^{-2m - 13}}{K}\cdot{L_{\min}\cdot(\Delta_{k, \ell}^{m - 1})^{-2}}T^m}\\
        \le& \frac{2C^m}{L_{\min}T^m}
    \end{align*}

    \eqref{eq:B_ineq_2} holds for \(-\frac{B_{k}^m}{{w_{k, \ell}^m}L_kT^m}\) with the same probability, Therefore we have:
    \begin{align}
        Pr\left[\left|\frac{B_{k, \ell}^m}{{w_{k, \ell}^m}L_kT^m}\right| \ge \frac{2C^m}{L_{\min}T^m} + \sqrt{\frac{\ln{\frac{4}{\beta}}}{16\cdot{w_{k, \ell}^m}L_kT^m}}\right] \le \frac{\beta}{2}.
    \end{align}
    \item
    \paragraph{Combining the two small probability bounds} Combining \eqref{eq:A_ineq} and \eqref{eq:B_ineq_2}, we have:
    \begin{align*}
        Pr\left[\left|r_{k, \ell}^m - \mu_k\right| \geq \frac{2C^m}{L_{\min}T^m} + \sqrt{\frac{4\ln{\frac{4}{\beta}}}{{w_{k, \ell}^m}L_kT^m}}\right] \le \beta.
    \end{align*}
    Then we set \(\beta = \delta/(2K\log{T})\), as this value satisfies \(\beta \geq 4e^{-\lambda/16}\), we have:
    \begin{align*}
        \sqrt{\frac{4\ln{\frac{4}{\beta}}}{{w_{k, \ell}^m}L_kT^m}} & \le \sqrt{\frac{4\ln{4/\beta}}{2^{10}(\Delta_{k, \ell})^{-2}\ln{\left(4/\beta\right)}}} \le \frac{\Delta_{k, \ell}^m}{16}.
    \end{align*}
    
    Therefore we have:
    \begin{align*}
        Pr\left[\left|r_{k, \ell}^m - \mu_k\right| \geq \frac{2C^m}{L_{\min}T^m} + \frac{\Delta_{k, \ell}^m}{16}\right] \le \delta/(2KL\log{T}).
    \end{align*}
    \item

    \begin{claim}
        \begin{align}
            \Pr\left[\Tilde{n}_{k, \ell}^m \geq 2 \cdot p_{k, \ell}^mT^m\right] \le \beta
        \end{align}
    \end{claim}

    \begin{proof}
        Again, we can use a Chernoff-Hoeffding bound on the r.v. representing the actual number of pulls of arm \(k\) by agent \(\ell\) in eopch \(m\), which is: \(\tilde{n}_{k, \ell}^m \coloneqq \sum_{t \in E^m} Y_{k, \ell}^t\). The expected value of this r.v. is \(\mathbb{E}\left[\tilde{n}_{k, \ell}^m\right] = {p_{k, \ell}^mT^m}\). The probability that the r.v. is more than twice its expectation is at most: \(2\exp{\left(-\frac{p_{k, \ell}^mT^m}{3}\right)} \le 2\exp{\left(-\frac{\lambda}{3}\right)} \le 4\exp{\left(\frac{-\lambda}{16}\right)} \le \beta\).
    \end{proof}

    Therefore using a union bound, we have:
    \begin{align}\label{eq:lem2_1_MA}
        \Pr\bigg[ & \forall \ell\in \mathcal{L}, \forall k\in \mathcal{K}_\ell, m \in \{1, \dots, M\}: \left|r_{k, \ell}^m - \mu_k\right| \geq \frac{2C^m}{L_{\min}T^m} + \frac{\Delta_{k,\ell}^m}{16} \text{ and } \Tilde{n}_{k, \ell}^m \geq 2 \cdot p_{k, \ell}^mT^m \bigg] \le \delta.
    \end{align}

\end{proof}

% \textbf{Lemma 5.}
\begin{lemma}
    Suppose that \(\mathcal{E}\) happens, then for all epochs \(m\):
    \label{lem:3}
    \begin{align}
        -\frac{2C^m}{L_{\min}T^m} - \frac{\Delta^{m - 1}_{k_\ell^{*}, \ell}}{8} \le r_{\max, \ell}^m - \mu^{*}_\ell \le \frac{2C^m}{L_{\min}T^m}.
    \end{align}
\end{lemma}

\begin{proof}
    We know that \(r_{\max, \ell}^m = \max_{k\in\mathcal{K}_\ell}{\{r_{k, \ell}^m - \frac{1}{16}\Delta_{k, \ell}^{m - 1}\}}\) and therefore \(r_{\max, \ell}^m \geq r_{k_\ell^{*}, \ell}^m - \frac{1}{16}\Delta_{k_\ell^{*}, \ell}^{m - 1}\), where \(k_\ell^*\) is the arm with the largest local mean, \(\mu^*_\ell\). Considering the lower bound implied by \(\mathcal{E}\) we have:
    \begin{align*}
        -\frac{2C^m}{L_{\min}T^m} - \frac{\Delta_{k_\ell^{*},\ell}^{m - 1}}{8} \le r_{\max, \ell}^m - \mu^{*}_\ell,
    \end{align*}
    and thus proving the lower bound of Lemma~\ref{lem:3}.
    We also know that
    \begin{align*}
        r_{\max, \ell}^m & = \max_{k\in\mathcal{K}_\ell}{\{r_{k, \ell}^m - \frac{1}{16}\Delta_{k, \ell}^{m - 1}\}}\\
        & \le \max_{k\in\mathcal{K}_\ell}\{\frac{2C^m}{L_{\min}T^m} + \frac{1}{16}\Delta_{k, \ell}^{m - 1} + \mu_{k} - \frac{1}{16}\Delta_{k, \ell}^{m - 1}\} \\
        & \le \mu^{*}_\ell + \frac{2C^m}{L_{\min}T^m}.
    \end{align*}
    Therefore the lemma is proved.
\end{proof}

Then we define \(\rho^m\) as:
\begin{equation}\label{eq:rho_MA}
    \rho^m = \sum_{s = 1}^{m}\frac{2C^s}{8^{m - s - 1}L_{\min}T^s}.
\end{equation}

\begin{lemma}
    \label{lem:4}
    Suppose that \(\mathcal{E}\) happens, then for all epochs \(m\) and arm \(k\):
    \begin{align}
        \Delta_{k, \ell}^m \le 2(\Delta_{k, \ell} + \sqrt{\frac{L_{\min}}{L}}2^{-m} + \rho^m + 2^{-4}).
    \end{align}
\end{lemma}

\begin{proof}
    We prove this by induction on \(m\). For \(m = 1\), the claim is true for all \(k\) and \(\ell\), as \(\Delta_{k, \ell}^m = 1 \le 2\cdot2^{-m}\).

    Suppose that the claim holds for \(m - 1\). Using Lemma~\ref{lem:3} and the definition  of the event \(\mathcal{E}\), we write:
    \begin{equation}
        % \resizebox{.94\linewidth}{!}{
        \begin{split}
            \quad r_{\max, \ell}^m - r_{k, \ell}^m  + 3 \cdot 2^{-7} =& (r_{\max, \ell}^m - \mu^{*}_\ell) + (\mu^{*}_\ell - \mu_k) + (\mu_k - r_{k, \ell}^m) + 3 \cdot 2^{-7}\\
            \overset{(a)}\le& \frac{2C^m}{L_{\min}T^m} + \Delta_{k, \ell} + \frac{2C^m}{L_{\min}T^m} + \frac{\Delta^{m - 1}_{k, \ell}}{16} + 3 \cdot 2^{-7}\\
            \overset{(b)}\le& \Delta_{k, \ell} + \frac{4C^m}{L_{\min}T^m}\\
            &+ \frac{1}{8}\left(\Delta_{k, \ell} + \sqrt{\frac{L_{\min}}{L}}2^{-(m-1)} + \sum_{s = 1}^{m-1}\frac{2C^{s}}{8^{(m - 1) - s - 1}L_{\min}T^s} + 2^{-4}\right) + 3 \cdot 2^{-7}\\
            \le& 2\left(\Delta_{k, \ell} + \sqrt{\frac{L_{\min}}{L}}2^{-m} + \sum_{s = 1}^{m}\frac{2C^{s}}{8^{m - s - 1}L_{\min}T^s} + 2^{-7}\right) + 3 \cdot 2^{-7}\\
            \le& 2(\Delta_{k, \ell} + \sqrt{\frac{L_{\min}}{L}}2^{-m} + \rho^m + 2^{-4}) 
        \end{split}
        % }
    \end{equation}
    where inequality (a) holds because \((r_{\max, \ell}^m - \mu_{*})\) is bound using Lemma~\ref{lem:3}, and \((\mu_k - r_{k, \ell}^m)\) is bound following the assumption of event \(\mathcal{E}\). Also, inequality (b) holds due to the induction's hypothesis. Finally, we know that \(\Delta_{k, \ell}^m = \max\{2^{-3}, r_{\max, \ell}^m - r_{k, \ell}^m + 3 \cdot 2^{-7}\}\). Therefore
    \begin{align*}
        \Delta_{k, \ell}^m \le 2\left(\Delta_{k, \ell} + \sqrt{\frac{L_{\min}}{L}}2^{-m} + \rho^m + 2^{-4}\right)
    \end{align*}
    Therefore Lemma~\ref{lem:4} is proved.
\end{proof}

\begin{lemma}
    \label{lem:5}
    Suppose that \(\mathcal{E}\) happens, then for all epochs \(m\) and arm \(k\):
    \begin{align}
        \Delta_{k, \ell}^m \ge \frac{1}{2} \Delta_{k, \ell} - 3\rho^m -\frac{3}{4}\sqrt{\frac{L_{\min}}{L}}2^{-m}.
    \end{align}
\end{lemma}

\begin{proof}
    \begin{align*}
        r_{\max, \ell}^m - r_{k, \ell}^m
         & \overset{(a)}\geq \left(\mu_{*} - \frac{2C^m}{L_{\min}T^m} - \frac{\Delta^{m - 1}_{k^{*}_\ell, \ell}}{8}\right) - \left(\mu_k + \frac{2C^m}{L_{\min}T^m} + \frac{\Delta^{m - 1}_{k, \ell}}{16}\right)\\
         & \overset{(b)}\geq \Delta_{k, \ell} - \frac{4C^m}{L_{\min}T^m} - \left(\frac{3}{8}\rho^{m-1} + \frac{3}{8}\sqrt{\frac{L_{\min}}{L}}2^{-(m - 1)} + \frac{1}{8}\Delta_{k, \ell} + 3\cdot2^{-7}\right) \\
         & \geq \frac{1}{2} \Delta_{k, \ell} - 3\sum_{s = 1}^{m}\frac{2C^s}{8^{m - s - 1}L_{\min}T^s} -\frac{3}{4}\sqrt{\frac{L_{\min}}{L}}2^{-m} -3\cdot2^{-7}.
    \end{align*}
    where (a) holds because \(r_{\max, \ell}^m\) can be bound using Lemma~\ref{lem:3}, and \(r_{k, \ell}^m\) is bound due to event \(\mathcal{E}\), and (b) is by substituting the two gaps by their lower bounds in Lemma~\ref{lem:4}. Since \(\Delta_{k, \ell}^m = \max\{2^{-3}, r_{\max, \ell}^m - r_{k, \ell}^m + 3 \cdot 2^{-7}\}\), we can infer that \(\Delta_{k, \ell}^m  \ge \frac{1}{2} \Delta_{k, \ell} - 3\rho^m -\frac{3}{4}\sqrt{\frac{L_{\min}}{L}}2^{-m}\), proving Lemma~\ref{lem:5}.
\end{proof}

\subsection{Main Proof for Theorem~\ref{thm:main-result}}

We decompose the total regret of each agent across epochs and within epochs, to the arms pulled. There are \(\tilde{n}^m_{k, \ell}\) pulls of arm \(k\) in epoch \(m\) by agent \(\ell\), each causing a regret of \((\mu_{\ell}^* - \mu_k)\). Thus the total regret can be written as:
\begin{align*}
    \sum_{m = 1}^{M}{\sum_{\ell \in \mathcal{L}}\sum_{k \in \mathcal{K}_\ell}{(\mu^*_{\ell} - \mu_k)\tilde{n}^m_{k, \ell}}} & \overset{(a)}\le 2\sum_{m = 1}^{M}{\sum_{\ell \in \mathcal{L}}\sum_{k \in \mathcal{K}_\ell}{(\mu^*_{\ell} - \mu_k){p^m_{k, \ell}T^m}}}\\
    & = 2\sum_{m = 1}^{M}{\sum_{\ell \in \mathcal{L}}\sum_{k \in \mathcal{K}_\ell}{\Delta_{k, \ell}{p^m_{k, \ell}T^m}}},
\end{align*}
where inequality (a) holds under event \(\mathcal{E}\) defined in~\eqref{eq:event-E_MA}.

Fix an epoch \(m\), an agent \(\ell\), and an arm \(k\) and denote \(\Reg^m_{k, \ell} := \Delta_{k, \ell}{p^m_{k, \ell}T^m}\) as the approximate regret for this \((m ,\ell, k)\) triad. Consider the three following cases:

\paragraph{Case 1} \(0 \le \Delta_{k, \ell} \le \frac{4}{2^{m}}\sqrt{\frac{L_{\min}}{L}}\).
In this case, if \(k \in \mathcal{A}^m_\ell\), implying \(p_{k,\ell}^m \le \frac{3}{4\abs{\mathcal{A}^m_\ell}}\) due to Lemma~\ref{lem:1}, we can upper bound the regret as follows,
\begin{align*}
    \Reg_{k, \ell}^m
    \le \Delta_{k, \ell} \cdot \frac{3}{4|\mathcal{A}^m_\ell|} \cdot \lambda \frac{K}{L_{\min}}2^{2(m - 1)}
    \overset{(a)}\le \Delta_{k, \ell} \frac{3}{4|\mathcal{A}^m_\ell|} \lambda \frac{K}{L_{\min}}\frac{16}{(\Delta_{k, \ell})^2}\frac{L_{\min}}{L} \le 16 \frac{K}{L|\mathcal{A}^m_\ell|} \frac{\lambda}{\Delta_{k, \ell}},
\end{align*}
where inequality (a) holds by the Case 1 assumption, i.e., \(2^m \le \frac{4}{\Delta_k}\sqrt{\frac{L_{\min}}{L}}\).

Therefore, the regret of Case 1 for arms in set \(\mathcal{A}^m_{\ell}\), for all agents \(\ell \in \mathcal{L}\) over all epochs \(m\) can be bounded as
\begin{align}
    \Reg_{\mathcal{A}, 1}
     & \le 16\sum_{m = 1}^M \sum_{\ell \in \mathcal{L}}\sum_{k \in \mathcal{A}^m_\ell}\frac{K}{|\mathcal{A}^m_\ell|} \frac{\lambda}{L\Delta_{k, \ell}} \nonumber\\
     & = O\left(\frac{K}{L}\log(\frac{8KL\log T}{\delta})\sum_{m = 1}^M
    \sum_{\ell \in \mathcal{L}}\sum_{k \in \mathcal{A}_\ell^m}\frac{1}{\Delta_{k, \ell}}\frac{1}{|\mathcal{A}_\ell^m|}\right)\nonumber\\
    & \overset{(a)}\le  O\left(\log T \log(\frac{8KL\log T}{\delta})\frac{K}{\Delta_{\min}}\right)
\end{align}
where (a) is by substituting \(\lambda\) by its definition.

And if \(k \in \mathcal{B}_\ell^m\), implying \(p_{k,\ell}^m \le 2^{-2m - 7}\frac{L_{\min}}{KL_k}\), we have
\begin{align*}
    \Reg_{k, \ell}^m \le \Delta_{k, \ell} \cdot 2^{-2m - 7}\frac{L_{\min}}{KL_k} \cdot \lambda \frac{K}{L_{\min}}2^{2(m - 1)} \le 2^3\frac{\lambda}{L_{k}}2^{-m}.
\end{align*}
Therefore, the regret of Case 1 for arms in set \(\mathcal{B}^m_{\ell}\), for all agents \(\ell \in \mathcal{L}\) over all epochs \(m\) can be bounded as
\begin{align}
    \Reg_{\mathcal{B}, 1} & \le 2^4 \sum_{m = 1}^M \sum_{\ell \in \mathcal{L}}\sum_{k \in \mathcal{B}_\ell^m} \frac{\lambda}{L_{k}}2^{-m} \nonumber\\
    & = O\left(K\log(\frac{8KL\log T}{\delta})\log T\right).
\end{align}

\paragraph{Case 2} \(\Delta_{k, \ell} > \frac{4}{2^{m}}\sqrt{\frac{L_{\min}}{L}}\) and \(\rho^{m - 1} < \Delta_{k, \ell}/32\). In this case we can have,
\begin{align*}
    r_{\max, \ell}^{m - 1} - r_{k, \ell}^{m - 1} & \overset{(a)}\geq \frac{1}{2} \Delta_{k, \ell} - 3\rho^{m-2} -\frac{3}{4}\sqrt{\frac{L_{\min}}{L}}2^{-(m-1)} -3 \cdot 2^{-7}\\
    & \geq \Delta_{k, \ell} \left( \frac{1}{2} - \frac{3}{32} - \frac{3}{8}\sqrt{\frac{L_{\min}}{L}}\sqrt{\frac{L}{L_{\min}}}\right) -3 \cdot 2^{-7} \\
    & \ge   \frac{\Delta_{k, \ell}}{32}-3 \cdot 2^{-7} \ge \frac{1}{2^{m + 3}}\sqrt{\frac{L_{\min}}{L}} -3 \cdot 2^{-7},
\end{align*}

where inequality (a) holds for Lemma~\ref{lem:5}. The above inequality \(r_{\max}^{m - 1} - r_{k}^{m - 1} \ge \frac{1}{2^{m + 3}}\sqrt{\frac{L_{\min}}{L}} -3 \cdot 2^{-7}\) implies that arm \(k \in \mathcal{B}^m_\ell\) for all \(\ell \in \mathcal{L}_k\).
We also have \(\Delta_{k, \ell}^{m - 1} = \max\{2^{-3}, r_{\max, \ell}^{m - 1} - r_{k, \ell}^{m - 1} + 3 \cdot 2^{-7}\} \ge r_{\max, \ell}^{m - 1} - r_{k, \ell}^{m - 1} + 3 \cdot 2^{-7} \ge \frac{\Delta_{k, \ell}}{32}-3 \cdot 2^{-7}+ 3 \cdot 2^{-7}\).
% \begin{align*}
%     \Delta_k^{m-1} \ge \sqrt{\frac{L_k}{L_{\min}}}\left(r_{\max}^m - r_{k}^m\right) \ge \frac{\Delta_{k}}{32}\sqrt{\frac{L_k}{L_{\min}}} 
% \end{align*}
Therefore we can bound \(\mathcal{R}_{k, \ell}^m\) as:
\begin{align*}
    \Reg_{k, \ell}^m \le \Delta_{k, \ell} p_{k, \ell}^mT^m \le \Delta_{k, \ell} \cdot \frac{2^{-2m}(\Delta_{k, \ell}^{m-1})^{-2}}{K_\ell}\frac{L_{\min}}{L_k}\frac{K_\ell}{K}\cdot\lambda \frac{K}{L_{\min}} 2^{2(m - 1)} \le \frac{2^{6}\lambda}{\Delta_{k, \ell}L_k}
\end{align*}

Therefore

\begin{align}
    \Reg_{\mathcal{B}, 2} & \le 2^{6}\sum_{m = 1}^M \sum_{\ell \in \mathcal{L}}\sum_{k \in \mathcal{B}_\ell^m}\frac{\lambda}{L_k\Delta_{k, \ell}}\nonumber\\
     & = O\left(\log\left(\frac{8KL\log T}{\delta}\right)\sum_{m=1}^M\sum_{\ell \in \mathcal{L}}\sum_{k \in \mathcal{B}^m_\ell}\frac{1}{\Delta_{k, \ell}}\frac{1}{L_k}\right) \nonumber\\
     & \le O\left(\log\left(\frac{8KL\log T}{\delta}\right)\sum_{m=1}^M\sum_{k \in \mathcal{K}} \sum_{\substack{\ell \in \mathcal{L}\\ \text{s.t. } k \in \mathcal{B}^m_{\ell}}}\frac{1}{\Delta_{k, \ell}}\frac{1}{L_k}\right) \nonumber\\
     & = O\left(\log T \log\left(\frac{8KL\log T}{\delta}\right) \sum_{k \in \mathcal{K}}\frac{1}{\Delta_{k, \min}}\right)
     \end{align}
\paragraph{Case 3} \(\Delta_{k, \ell} > \frac{4}{2^m}\sqrt{\frac{L_{\min}}{L}}\) and \(\rho^{m - 1} \geq \Delta_{k}/32\). Therefore \(\Delta_{k, \ell} \le 32 \rho^{m - 1}\). Hence, if \(k \in \mathcal{B}_\ell^{m - 1}\), we have
\begin{align*}
    \Reg_{k, \ell}^m \le \Delta_{k, \ell}p_{k, \ell}^mT^m \le 32\rho^{m - 1}\cdot\frac{2^6}{K_\ell}\frac{L_{\min}}{L_k}\frac{K_\ell}{K}\cdot\lambda \frac{K}{L_{\min}}2^{2(m -1)}
    % {\lambda\Delta_{k}2^{2(m - 1)}} \le {8\frac{\lambda}{L_k}\rho^{m-1}2^{2m}}.
\end{align*}

Therefore
\begin{align}
    \Reg_{\mathcal{B}, 3} & \le 2^9\sum_{m = 1}^M\sum_{\ell \in \mathcal{L}}\sum_{k \in \mathcal{B}^m_\ell}\frac{\lambda}{L_k}\rho^{m-1}2^{2m} = 2^9\lambda\sum_{m = 1}^M\sum_{\ell \in \mathcal{L}}\sum_{k \in \mathcal{B}^m_\ell}\frac{1}{L_k}2^{2m}\rho^{m - 1}\nonumber \\
     & \overset{\eqref{eq:rho_MA}}= 2^9\lambda\sum_{m = 1}^M\sum_{\ell \in \mathcal{L}}\sum_{k \in \mathcal{B}^m_\ell}\frac{1}{L_k}2^{2m}\sum_{s = 1}^{m-1}\frac{2C^{s}}{8^{m - s - 5}L_{\min}T^s}\\
     &\le {2^9\lambda}\sum_{\ell \in \mathcal{L}}\sum_{k \in \mathcal{K}}\sum_{m = 1}^M\frac{1}{L_k}\sum_{s = 1}^{m-1}\frac{2^{2m + 1}C^{s}}{8^{m - s - 5}L_{\min}T^s}\nonumber\\
     & \overset{(a)} = 2^9\lambda\sum_{\ell \in \mathcal{L}}\sum_{k \in \mathcal{K}}\sum_{s = 1}^{M - 1}\frac{C^s}{L_k} \sum_{m = s + 1}^{M} \frac{2^{2m + 1}L_{\min}}{8^{m - 5 - s}L_{\min}\lambda K 2^{2(s - 1)}} \nonumber                               \\
     &= 2^9\sum_{\ell \in \mathcal{L}}\sum_{k \in \mathcal{K}}\sum_{s = 1}^{M - 1}\frac{C^s}{KL_k} \sum_{m = s+1}^{M} \frac{4^{m - 1 - s}}{8^{m - 5 - s}}\nonumber                                                \\
     & \le 2^9\sum_{\ell \in \mathcal{L}}\sum_{k \in \mathcal{K}}\sum_{s = 1}^{M - 1}\frac{C^s}{KL_k} \le 2^9 \frac{L}{L_{\min}}\sum_{s = 1}^{M - 1}C^s\label{eq:case3_B},
\end{align}
where (a) is because \(T^m = \lambda \frac{K}{L_{\min}} 2^{2(m - 1)}\), and changing the order of summations.

On the other hand if \(k \in \mathcal{A}^m_\ell\), we have
\begin{align*}
    \Reg_{k, \ell}^m \le \Delta_{k, \ell}p_{k, \ell}^mT^m \le 32\rho^{m - 1}\cdot\frac{1}{|\mathcal{A}_\ell^m|}\cdot\lambda \frac{K}{L_{\min}}2^{2(m -1)}
    % {\lambda\Delta_{k}2^{2(m - 1)}} \le {8\frac{\lambda}{L_k}\rho^{m-1}2^{2m}}.
\end{align*}

Therefore following the same analysis leading to \eqref{eq:case3_B}, we have
\begin{align}
    \Reg_{\mathcal{A}, 3}
     & \le 2^3\frac{K}{L_{\min}}\sum_{m = 1}^M\sum_{\ell \in \mathcal{L}}\sum_{k \in \mathcal{A}^m_\ell}\frac{2^{2m}}{|\mathcal{A}^m_\ell|}\rho^{m-1} \nonumber \\
     & \overset{\eqref{eq:rho_MA}}= 2^3\frac{K}{L_{\min}}\sum_{\ell \in \mathcal{L}}\sum_{m = 1}^M \sum_{s = 1}^{m-1}\frac{2^{2m+1}C^{s}}{8^{m - s - 1}L_{\min}T^s} \nonumber\\
     &= 2^3\frac{KL}{L_{\min}}\sum_{s = 1}^{M - 1} C^s \sum_{m = s + 1}^{M} \frac{2^{2m}L_{\min}}{8^{m - s - 1} L_{\min}K 2^{2(s - 1)}} \nonumber                                           \\
     & \overset{(a)}= 2^7 \frac{L}{L_{\min}}\sum_{s = 1}^{M - 1} C^s \sum_{m = s+1}^{M} \frac{4^{m - s - 1}}{8^{m - s - 1}} \le 2^7 \frac{L}{L_{\min}}\sum_{s = 1}^{M - 1} C^s\label{eq:case3_A},
\end{align}

Hence, the total aggregated regret can be bounded as follows, with probability at least \(1 - \delta\).
\begin{align}
    &\Reg_{\mathcal{A}, 1} + \Reg_{\mathcal{B}, 1} + \Reg_{\mathcal{B}, 2} + \Reg_{\mathcal{A}, 3} + \Reg_{\mathcal{B}, 3} \le O\left(\frac{L}{L_{\min}}C
    + \log T \log(\frac{8KL\log T}{\delta})\frac{K}{\Delta_{\min}}\right).
\end{align}
\newpage
\section{Proof for Remark~\ref{remark:1}}
\label{appendix:2}
\subsection{Lemma}
\begin{lemma}
    
\end{lemma}
\begin{proof}
    \item
    \paragraph{Decomposing \(\pmb{r_{k, \ell}^m}\)} For this proof, we condition on all random variables before some fixed epoch \(m\), so that \(p^m_{k, \ell}\) is a deterministic quantity. At each step in this epoch, agent \(\ell\) picks arm \(k \in \mathcal{K}_\ell\) with probability \(p_{k, \ell}^m\). Let \(Y_{k, \ell}^t\) be an indicator for arm \(k\) being pulled in step \(t\) by agent \(\ell\). Let \(R_{k, \ell}^t\) be the stochastic reward of arm \(k\) on step \(t\) for agent \(\ell\). and \(c^t_{k, \ell} \coloneqq \tilde{R}_{k, \ell}^t - R_{k, \ell}^t\) be the corruption added to arm \(k\) by the adversary in step \(t\) for agent \(\ell\). Note that \(c^t_{k, \ell}\) may depend on all the stochastic rewards up to (and including) round \(t\), and also on all previous choices of the agents (though not the choice at step \(t\)). Finally, we denote \(E^m\) to be the \(T^m\) timesteps in epoch \(m\), then we can write the random variable \(r_{k, \ell}^m\) defined in Line~\ref{line:approx_mean_reward} of Algorithm~\ref{alg:alg} as follows
    \begin{equation}
        r_{k, \ell}^m = \frac{\sum_{\ell' \in \mathcal{L}_k}\sum_{t \in E^m}Y_{k, \ell'}^t\left(R_{k, \ell'}^t + c_{k, \ell'}^t\right)}{\sum_{\ell' \in \mathcal{L}_k}{p_{k, \ell'}^m} T^m}.
    \end{equation}
    Let us define two auxiliary random variables \(A_{k, \ell}^m\) and \(B_{k, \ell}^m\) as
    \begin{align*}
        A_{k, \ell}^m & = {\sum_{\ell' \in \mathcal{L}_k}\sum_{t \in E^m}Y_{k, \ell'}^t R_{k, \ell'}^t}\\
        B_{k, \ell}^m & = {\sum_{\ell' \in \mathcal{L}_k}\sum_{t \in E^m}Y_{k, \ell'}^t c_{k, \ell'}^t},
    \end{align*}
    \item
    \paragraph{Bounding \(\pmb{A_{k, \ell}^m}\)} To bound \(A_{k, \ell}^m\), we first calculate its expected value as
    \begin{align*}
        \mathbb{E}\left[A_{k, \ell}^m\right] = {\sum_{\ell' \in \mathcal{L}_k}\sum_{t \in E^m}\mathbb{E}\left[Y_{k, \ell'}^t R_{k, \ell'}^t\right]} = \sum_{\ell' \in \mathcal{L}_k}T^m\left(p_{k, \ell'}^m\mu_k\right) = T^m \mu_k \sum_{\ell' \in \mathcal{L}_k} p_{k, \ell'}^m,
    \end{align*}
    Therefore, according to the Chernoff-Hoeffding bound \cite{dubhashi2009concentration}, and because \(\mu_k \le 1\), for \(\beta \le 1\) we have:
    \begin{align}
    \label{eq:A_ineq}
        Pr\left[\left|\frac{A_{k, \ell}^m}{T^m\sum_{\ell' \in \mathcal{L}_k} p_{k, \ell'}^m} - \mu_k\right| \ge \sqrt{\frac{3\ln{\frac{4}{\beta}}}{T^m\sum_{\ell' \in \mathcal{L}_k} p_{k, \ell'}^m}}\right] \le \frac{\beta}{2}.
    \end{align}

     \item
    \paragraph{Bounding \(\pmb{B_{k, \ell}^m}\)}
    Next, we aim to bound \(B_{k, \ell}^m\). To this end, we define the sequence of r.v.s, \(\{X_{t'}\}_{t'\in[1,\dots, L_kT]}\), where \(X_{t'} = \left(Y_{k, \ell}^t - p_{k, \ell}^m\right)\cdot  c_{k, \ell}^t\) for all \(t'\), where \(t = \lceil  t'/L_k \rceil\), and \(\ell = t' \text{ mod } t\). Then \(\{X_{t'}\}_{t'\in[1,\dots, L_kT]}\) is a martingale sequence with respect to the filtration \(\{\mathcal{F}_{t'}\}_{t'=1}^{L_kT}\) generated by r.v.s \(\{Y_{k, h}^s\}_{k \in \mathcal{K}, h \in \mathcal{L}_k, s \le t}\) and \(\{R_{k, h}^s\}_{k \in \mathcal{K}, h \in \mathcal{L}_k, s \le t}\). The corruption \(c_{k, \ell}^m\) becomes deterministic conditioned on \(\mathcal{F}_{t-1}\), and since \(\mathbb{E}\left[Y_{k, \ell}^t | \mathcal{F}_{t-1}\right] = p_{k, \ell}^m\), we have:
    \begin{align*}
        \mathbb{E}\left[X_{t'} | \mathcal{F}_{t-1}\right] = \sum_{\ell' \in \mathcal{L}}\sum_{t \in E^m}\mathbb{E}\left[Y_{k, \ell'}^t - p_{k, \ell'}^m | \mathcal{F}_{t-1}\right]\cdot c_{k, \ell'}^t = 0.
    \end{align*}
    Using a Freedman-type concentration inequality introduced in \citeauthor{beygelzimer2011contextual}, the predictive quadratic variation of this martingale can be bounded as:
    \begin{gather*}
        \Pr\left[\sum_{t'} X_{t} \geq \frac{V}{b} + b\ln \frac{4}{\beta}\right] \le \frac{\beta}{4},
    \end{gather*}
    where \(\left|X_{t'}\right| \le b\), and
    \begin{align*}
        V = & \sum_{\ell' \in \mathcal{L}}\sum_{t \in E^m}\mathbb{E}\left[X_{t'}^2 | \mathcal{F}_{t-1}\right] \le \sum_{\ell' \in \mathcal{L}}\sum_{t \in E^m} \left|c_{k, \ell'}^t\right|\text{Var}\left(Y_{k, \ell'}^t\right) \le  \sum_{\ell' \in \mathcal{L}} p_{k, \ell'}^m \sum_{t \in E^m} \left|c_{k, \ell'}^t\right|.
    \end{align*}
    Therefore, as \(|X_{t'}| \le 1\) we obtain that with a probability of \(\frac{\beta}{4}\):
    \begin{align}
        \frac{B_{k, \ell}^m}{T^m\sum_{\ell' \in \mathcal{L}_k} p_{k, \ell'}^m} & \overset{(a)}\ge \frac{\sum_{\ell' \in \mathcal{L}} p_{k, \ell'}^m\sum_{t \in E^m} c_{k, \ell'}^t}{T^m\sum_{\ell' \in \mathcal{L}_k} p_{k, \ell'}^m} + \frac{\left(V + \ln{(4/\beta)}\right)}{T^m\sum_{\ell' \in \mathcal{L}_k} p_{k, \ell'}^m} \nonumber\\
        & \ge 2\frac{\sum_{\ell' \in \mathcal{L}} p_{k, \ell'}^m\sum_{t \in E^m} \left|c_{k, \ell'}^t\right|}{T^m\sum_{\ell' \in \mathcal{L}_k} p_{k, \ell'}^m} + \frac{\ln{(4/\beta)}}{T^m\sum_{\ell' \in \mathcal{L}_k} p_{k, \ell'}^m},\label{eq:B_ineq_1}
    \end{align}

    where (a) is because \(\sum_{t'}X_{t'} = \sum_{\ell' \in \mathcal{L}_k}\sum_{t \in E^m}(Y_{k, \ell'}^t - p_{k, \ell'}^m)\cdot c_{k, \ell'}^t = B_{k}^m - \sum_{\ell' \in \mathcal{L}}\sum_{t \in E^m} c_{k, \ell'}^t\).

    We also define the total corruption on agent \(\ell\)'s rewards in epoch \(m\) as \(C^m_\ell = \sum_{t \in E^m}\max_{k \in \mathcal{K}} |c_{k, \ell}^t|\), therefore we can conclude that \(\sum_{\ell \in \mathcal{L}} \left|c_{k, \ell}^t\right| \le C^m\). We choose a parameter \(\beta \geq 4e^{(-\lambda/16)}\), therefore
    \begin{align}
        T^m\sum_{\ell' \in \mathcal{L}_k} p_{k, \ell'}^m \overset{(a)}{\ge} & T^m \sum_{\ell' \in \mathcal{L}_k} \frac{L_{\min}}{L_k}\frac{2^{-2m - 7}}{K}\nonumber\\
        \overset{(b)}{\ge} & 2^{-8} \left(\frac{2^{-2m}L_{\min}(\Delta_{k, \ell}^{m - 1})^{-2}}{K}\right)T^m\label{eq:lower_boumd_denom_1}\\
        \ge& 2^{-8} \left(\frac{2^{-2m}L_{\min}(\Delta_{k, \ell}^m)^{-2}}{K} \cdot \frac{2^{2(m - 1)}\lambda K}{L_{\min}}\right)\nonumber\\
        \ge & 2^{10}\ln(\frac{4}{\beta})
        \ge 16\ln{(\frac{4}{\beta})},
    \end{align}
    where in step (a) we are replacing the probability \(p_{k, \ell'}^m\) by it's smallest value according to Lemma~\ref{lem:1N}. In step (b) we are using the fact that the inverse of the estimated reward gap squared \(\Delta_{k, \ell}^{m - 1})^{-2}\) is upper bounded by \(2\).
    Now we can rewrite Equation~\eqref{eq:B_ineq_1} as:
   \begin{align}
        \label{eq:B_ineq_2}
        \frac{B_{k, \ell}^m}{T^m\sum_{\ell' \in \mathcal{L}_k} p_{k, \ell'}^m}\ge 2\frac{\sum_{\ell' \in \mathcal{L}} p_{k, \ell'}^m\sum_{t \in E^m} c_{k, \ell'}^t}{T^m\sum_{\ell' \in \mathcal{L}_k} p_{k, \ell'}^m} + \sqrt{\frac{\ln{4/\beta}}{16\cdot T^m\sum_{\ell' \in \mathcal{L}_k} p_{k, \ell'}^m}}.
    \end{align}
    In the next step, we aim to bound the corruption term of the upper bound stated in Equation~\eqref{eq:B_ineq_2}. 
    \begin{align}
       \frac{\sum_{\ell' \in \mathcal{L}} p_{k, \ell'}^m\sum_{t \in E^m} c_{k, \ell'}^t}{T^m\sum_{\ell' \in \mathcal{L}_k} p_{k, \ell'}^m} \le \frac{\sum_{\ell' \in \mathcal{L}} p_{k, \ell'}^mC^m_{\ell'}}{T^m\sum_{\ell' \in \mathcal{L}_k} p_{k, \ell'}^m} \overset{(a)}{\le} \frac{\max_{\ell' \in \mathcal{L}} C^m_{\ell'}}{T^m} \overset{(b)}{\le} \frac{C^m}{T^m}\label{eq:corr_upper_bound_nom}
    \end{align}
    Where step \(a\) is due to the fact that the above term is a weighted average over agents' corruption during epoch \(m\). We do not know how many of the agents have arm \(k\) in their active arm-set and how many in their bad arm-set, the best upper bound we can provide for this term is the max of \(C^m_{\ell'}\). As the corruption level is set by the adversary, this value can be equal to the total amount of corruption during epoch \(m\) (step~(b)).

    \eqref{eq:B_ineq_2} holds for \(-\frac{B_{k}^m}{{w_{k, \ell}^m}L_kT^m}\) with the same probability, Therefore we have:
    \begin{align}
        Pr\left[\left|\frac{B_{k, \ell}^m}{T^m\sum_{\ell' \in \mathcal{L}_k} p_{k, \ell'}^m}\right| \ge \frac{2C^m}{T^m} + \sqrt{\frac{\ln{\frac{4}{\beta}}}{16 T^m\sum_{\ell' \in \mathcal{L}_k} p_{k, \ell'}^m}}\right] \le \frac{\beta}{2}.
    \end{align}
    \item
    \paragraph{Combining the two small probability bounds} Combining \eqref{eq:A_ineq} and \eqref{eq:B_ineq_2}, we have:
    \begin{align*}
        Pr\left[\left|r_{k, \ell}^m - \mu_k\right| \geq \frac{2C^m}{T^m} + \sqrt{\frac{4\ln{\frac{4}{\beta}}}{T^m\sum_{\ell' \in \mathcal{L}_k} p_{k, \ell'}^m}}\right] \le \beta.
    \end{align*}
    Then we set \(\beta = \delta/(2K\log{T})\), as this value satisfies \(\beta \geq 4e^{-\lambda/16}\), we have:
    \begin{align*}
        \sqrt{\frac{4\ln{\frac{4}{\beta}}}{T^m\sum_{\ell' \in \mathcal{L}_k} p_{k, \ell'}^m}} & \le \sqrt{\frac{4\ln{4/\beta}}{2^{10}(\Delta_{k, \ell})^{-2}\ln{\left(4/\beta\right)}}} \le \frac{\Delta_{k, \ell}^m}{16}.
    \end{align*}
    
    Therefore we have:
    \begin{align*}
        Pr\left[\left|r_{k, \ell}^m - \mu_k\right| \geq \frac{2C^m}{T^m} + \frac{\Delta_{k, \ell}^m}{16}\right] \le \delta/(2KL\log{T}).
    \end{align*}
    \item

    \begin{claim}
        \begin{align}
            \Pr\left[\Tilde{n}_{k, \ell}^m \geq 2 \cdot p_{k, \ell}^mT^m\right] \le \beta
        \end{align}
    \end{claim}

    \begin{proof}
        Again, we can use a Chernoff-Hoeffding bound on the r.v. representing the actual number of pulls of arm \(k\) by agent \(\ell\) in eopch \(m\), which is: \(\tilde{n}_{k, \ell}^m \coloneqq \sum_{t \in E^m} Y_{k, \ell}^t\). The expected value of this r.v. is \(\mathbb{E}\left[\tilde{n}_{k, \ell}^m\right] = {p_{k, \ell}^mT^m}\). The probability that the r.v. is more than twice its expectation is at most: \(2\exp{\left(-\frac{p_{k, \ell}^mT^m}{3}\right)} \le 2\exp{\left(-\frac{\lambda}{3}\right)} \le 4\exp{\left(\frac{-\lambda}{16}\right)} \le \beta\).
    \end{proof}

    Therefore using a union bound, we have:
    \begin{align}\label{eq:lem2_1_MA}
        \Pr\bigg[ & \forall \ell\in \mathcal{L}, \forall k\in \mathcal{K}_\ell, m \in \{1, \dots, M\}: \left|r_{k, \ell}^m - \mu_k\right| \geq \frac{2C^m}{T^m} + \frac{\Delta_{k,\ell}^m}{16} \text{ and } \Tilde{n}_{k, \ell}^m \geq 2 \cdot p_{k, \ell}^mT^m \bigg] \le \delta.
    \end{align}

\begin{lemma}
    Suppose that \(\mathcal{E}\) happens, then for all epochs \(m\):
    \label{lem:3N}
    \begin{align}
        -\frac{2C^m}{T^m} - \frac{\Delta^{m - 1}_{k_\ell^{*}, \ell}}{8} \le r_{\max, \ell}^m - \mu^{*}_\ell \le \frac{2C^m}{T^m}.
    \end{align}
\end{lemma}

\begin{proof}
    We know that \(r_{\max, \ell}^m = \max_{k\in\mathcal{K}_\ell}{\{r_{k, \ell}^m - \frac{1}{16}\Delta_{k, \ell}^{m - 1}\}}\) and therefore \(r_{\max, \ell}^m \geq r_{k_\ell^{*}, \ell}^m - \frac{1}{16}\Delta_{k_\ell^{*}, \ell}^{m - 1}\), where \(k_\ell^*\) is the arm with the largest local mean, \(\mu^*_\ell\). Considering the lower bound implied by \(\mathcal{E}\) we have:
    \begin{align*}
        -\frac{2C^m}{T^m} - \frac{\Delta_{k_\ell^{*},\ell}^{m - 1}}{8} \le r_{\max, \ell}^m - \mu^{*}_\ell,
    \end{align*}
    and thus proving the lower bound of Lemma~\ref{lem:3N}.
    We also know that
    \begin{align*}
        r_{\max, \ell}^m & = \max_{k\in\mathcal{K}_\ell}{\{r_{k, \ell}^m - \frac{1}{16}\Delta_{k, \ell}^{m - 1}\}}\\
        & \le \max_{k\in\mathcal{K}_\ell}\{\frac{2C^m}{T^m} + \frac{1}{16}\Delta_{k, \ell}^{m - 1} + \mu_{k} - \frac{1}{16}\Delta_{k, \ell}^{m - 1}\} \\
        & \le \mu^{*}_\ell + \frac{2C^m}{T^m}.
    \end{align*}
    Therefore the lemma is proved.
\end{proof}

Then we define \(\rho^m\) as:
\begin{equation}\label{eq:rho_MA}
    \rho^m = \sum_{s = 1}^{m}\frac{2C^s}{8^{m - s - 1}T^s}.
\end{equation}

\begin{lemma}
    \label{lem:4N}
    Suppose that \(\mathcal{E}\) happens, then for all epochs \(m\) and arm \(k\):
    \begin{align}
        \Delta_{k, \ell}^m \le 2(\Delta_{k, \ell} + \sqrt{\frac{L_{\min}}{L}}2^{-m} + \rho^m + 2^{-4}).
    \end{align}
\end{lemma}

\begin{proof}
    We prove this by induction on \(m\). For \(m = 1\), the claim is true for all \(k\) and \(\ell\), as \(\Delta_{k, \ell}^m = 1 \le 2\cdot2^{-m}\).

    Suppose that the claim holds for \(m - 1\). Using Lemma~\ref{lem:3N} and the definition  of the event \(\mathcal{E}\), we write:
    \begin{equation}
        % \resizebox{.94\linewidth}{!}{
        \begin{split}
            \quad r_{\max, \ell}^m - r_{k, \ell}^m  + 3 \cdot 2^{-7} =& (r_{\max, \ell}^m - \mu^{*}_\ell) + (\mu^{*}_\ell - \mu_k) + (\mu_k - r_{k, \ell}^m) + 3 \cdot 2^{-7}\\
            \overset{(a)}\le& \frac{2C^m}{T^m} + \Delta_{k, \ell} + \frac{2C^m}{T^m} + \frac{\Delta^{m - 1}_{k, \ell}}{16} + 3 \cdot 2^{-7}\\
            \overset{(b)}\le& \Delta_{k, \ell} + \frac{4C^m}{T^m}\\
            &+ \frac{1}{8}\left(\Delta_{k, \ell} + \sqrt{\frac{L_{\min}}{L}}2^{-(m-1)} + \sum_{s = 1}^{m-1}\frac{2C^{s}}{8^{(m - 1) - s - 1}T^s} + 2^{-4}\right) + 3 \cdot 2^{-7}\\
            \le& 2\left(\Delta_{k, \ell} + \sqrt{\frac{L_{\min}}{L}}2^{-m} + \sum_{s = 1}^{m}\frac{2C^{s}}{8^{(m - 1) - s - 1}T^s} + 2^{-7}\right) + 3 \cdot 2^{-7}\\
            \le& 2(\Delta_{k, \ell} + \sqrt{\frac{L_{\min}}{L}}2^{-m} + \rho^m + 2^{-4}) 
        \end{split}
        % }
    \end{equation}
    where inequality (a) holds because \((r_{\max, \ell}^m - \mu_{*})\) is bound using Lemma~\ref{lem:3N}, and \((\mu_k - r_{k, \ell}^m)\) is bound following the assumption of event \(\mathcal{E}\). Also, inequality (b) holds due to the induction's hypothesis. Finally, we know that \(\Delta_{k, \ell}^m = \max\{2^{-3}, r_{\max, \ell}^m - r_{k, \ell}^m + 3 \cdot 2^{-7}\}\). Therefore
    \begin{align*}
        \Delta_{k, \ell}^m \le 2\left(\Delta_{k, \ell} + \sqrt{\frac{L_{\min}}{L}}2^{-m} + \rho^m + 2^{-4}\right)
    \end{align*}
    Therefore Lemma~\ref{lem:4N} is proved.
\end{proof}

\begin{lemma}
    \label{lem:5N}
    Suppose that \(\mathcal{E}\) happens, then for all epochs \(m\) and arm \(k\):
    \begin{align}
        \Delta_{k, \ell}^m \ge \frac{1}{2} \Delta_{k, \ell} - 3\rho^m -\frac{3}{4}\sqrt{\frac{L_{\min}}{L}}2^{-m}.
    \end{align}
\end{lemma}

\begin{proof}
    \begin{align*}
        r_{\max, \ell}^m - r_{k, \ell}^m
         & \overset{(a)}\geq \left(\mu_{*} - \frac{2C^m}{T^m} - \frac{\Delta^{m - 1}_{k^{*}_\ell, \ell}}{8}\right) - \left(\mu_k + \frac{2C^m}{T^m} + \frac{\Delta^{m - 1}_{k, \ell}}{16}\right)\\
         & \overset{(b)}\geq \Delta_{k, \ell} - \frac{4C^m}{T^m} - \left(\frac{3}{8}\rho^{m-1} + \frac{3}{8}\sqrt{\frac{L_{\min}}{L}}2^{-(m - 1)} + \frac{1}{8}\Delta_{k, \ell} + 3\cdot2^{-7}\right) \\
         & \geq \frac{1}{2} \Delta_{k, \ell} - 3\sum_{s = 1}^{m}\frac{2C^s}{8^{m - s - 1}T^s} -\frac{3}{4}\sqrt{\frac{L_{\min}}{L}}2^{-m} -3\cdot2^{-7}.
    \end{align*}
    where (a) holds because \(r_{\max, \ell}^m\) can be bound using Lemma~\ref{lem:3N}, and \(r_{k, \ell}^m\) is bound due to event \(\mathcal{E}\), and (b) is by substituting the two gaps by their lower bounds in Lemma~\ref{lem:4N}. Since \(\Delta_{k, \ell}^m = \max\{2^{-3}, r_{\max, \ell}^m - r_{k, \ell}^m + 3 \cdot 2^{-7}\}\), we can infer that \(\Delta_{k, \ell}^m  \ge \frac{1}{2} \Delta_{k, \ell} - 3\rho^m -\frac{3}{4}\sqrt{\frac{L_{\min}}{L}}2^{-m}\), proving Lemma~\ref{lem:5N}.
\end{proof}

\subsection{Main Proof for Theorem~\ref{thm:main-result}}

We decompose the total regret of each agent across epochs and within epochs, to the arms pulled. There are \(\tilde{n}^m_{k, \ell}\) pulls of arm \(k\) in epoch \(m\) by agent \(\ell\), each causing a regret of \((\mu_{\ell}^* - \mu_k)\). Thus the total regret can be written as:
\begin{align*}
    \sum_{m = 1}^{M}{\sum_{\ell \in \mathcal{L}}\sum_{k \in \mathcal{K}_\ell}{(\mu^*_{\ell} - \mu_k)\tilde{n}^m_{k, \ell}}} & \overset{(a)}\le 2\sum_{m = 1}^{M}{\sum_{\ell \in \mathcal{L}}\sum_{k \in \mathcal{K}_\ell}{(\mu^*_{\ell} - \mu_k){p^m_{k, \ell}T^m}}}\\
    & = 2\sum_{m = 1}^{M}{\sum_{\ell \in \mathcal{L}}\sum_{k \in \mathcal{K}_\ell}{\Delta_{k, \ell}{p^m_{k, \ell}T^m}}},
\end{align*}
where inequality (a) holds under event \(\mathcal{E}\) defined in~\eqref{eq:event-E_MA}.

Fix an epoch \(m\), an agent \(\ell\), and an arm \(k\) and denote \(\Reg^m_{k, \ell} := \Delta_{k, \ell}{p^m_{k, \ell}T^m}\) as the approximate regret for this \((m ,\ell, k)\) triad. Consider the three following cases:

\paragraph{Case 1} \(0 \le \Delta_{k, \ell} \le \frac{4}{2^{m}}\sqrt{\frac{L_{\min}}{L}}\).
In this case, if \(k \in \mathcal{A}^m_\ell\), implying \(p_{k,\ell}^m \le \frac{3}{4\abs{\mathcal{A}^m_\ell}}\) due to Lemma~\ref{lem:1}, we can upper bound the regret as follows,
\begin{align*}
    \Reg_{k, \ell}^m
    \le \Delta_{k, \ell} \cdot \frac{3}{4|\mathcal{A}^m_\ell|} \cdot \lambda \frac{K}{L_{\min}}2^{2(m - 1)}
    \overset{(a)}\le \Delta_{k, \ell} \frac{3}{4|\mathcal{A}^m_\ell|} \lambda \frac{K}{L_{\min}}\frac{16}{(\Delta_{k, \ell})^2}\frac{L_{\min}}{L} \le 16 \frac{K}{L|\mathcal{A}^m_\ell|} \frac{\lambda}{\Delta_{k, \ell}},
\end{align*}
where inequality (a) holds by the Case 1 assumption, i.e., \(2^m \le \frac{4}{\Delta_k}\sqrt{\frac{L_{\min}}{L}}\).

Therefore, the regret of Case 1 for arms in set \(\mathcal{A}^m_{\ell}\), for all agents \(\ell \in \mathcal{L}\) over all epochs \(m\) can be bounded as
\begin{align}
    \Reg_{\mathcal{A}, 1}
     & \le 16\sum_{m = 1}^M \sum_{\ell \in \mathcal{L}}\sum_{k \in \mathcal{A}^m_\ell}\frac{K}{|\mathcal{A}^m_\ell|} \frac{\lambda}{L\Delta_{k, \ell}} \nonumber\\
     & = O\left(\frac{K}{L}\log(\frac{8KL\log T}{\delta})\sum_{m = 1}^M
    \sum_{\ell \in \mathcal{L}}\sum_{k \in \mathcal{A}_\ell^m}\frac{1}{\Delta_{k, \ell}}\frac{1}{|\mathcal{A}_\ell^m|}\right)\nonumber\\
    & \overset{(a)}\le  O\left(\log T \log(\frac{8KL\log T}{\delta})\frac{K}{\Delta_{\min}}\right)
\end{align}
where (a) is by substituting \(\lambda\) by its definition.

And if \(k \in \mathcal{B}_\ell^m\), implying \(p_{k,\ell}^m \le 2^{-2m - 7}\frac{L_{\min}}{KL_k}\), we have
\begin{align*}
    \Reg_{k, \ell}^m \le \Delta_{k, \ell} \cdot 2^{-2m - 7}\frac{L_{\min}}{KL_k} \cdot \lambda \frac{K}{L_{\min}}2^{2(m - 1)} \le 2^3\frac{\lambda}{L_{k}}2^{-m}.
\end{align*}
Therefore, the regret of Case 1 for arms in set \(\mathcal{B}^m_{\ell}\), for all agents \(\ell \in \mathcal{L}\) over all epochs \(m\) can be bounded as
\begin{align}
    \Reg_{\mathcal{B}, 1} & \le 2^4 \sum_{m = 1}^M \sum_{\ell \in \mathcal{L}}\sum_{k \in \mathcal{B}_\ell^m} \frac{\lambda}{L_{k}}2^{-m} \nonumber\\
    & = O\left(K\log(\frac{8KL\log T}{\delta})\log T\right).
\end{align}

\paragraph{Case 2} \(\Delta_{k, \ell} > \frac{4}{2^{m}}\sqrt{\frac{L_{\min}}{L}}\) and \(\rho^{m - 1} < \Delta_{k, \ell}/32\). In this case we can have,
\begin{align*}
    r_{\max, \ell}^{m - 1} - r_{k, \ell}^{m - 1} & \overset{(a)}\geq \frac{1}{2} \Delta_{k, \ell} - 3\rho^{m-2} -\frac{3}{4}\sqrt{\frac{L_{\min}}{L}}2^{-(m-1)} -3 \cdot 2^{-7}\\
    & \geq \Delta_{k, \ell} \left( \frac{1}{2} - \frac{3}{32} - \frac{3}{8}\sqrt{\frac{L_{\min}}{L}}\sqrt{\frac{L}{L_{\min}}}\right) -3 \cdot 2^{-7} \\
    & \ge   \frac{\Delta_{k, \ell}}{32}-3 \cdot 2^{-7} \ge \frac{1}{2^{m + 3}}\sqrt{\frac{L_{\min}}{L}} -3 \cdot 2^{-7},
\end{align*}

where inequality (a) holds for Lemma~\ref{lem:5N}. The above inequality \(r_{\max}^{m - 1} - r_{k}^{m - 1} \ge \frac{1}{2^{m + 3}}\sqrt{\frac{L_{\min}}{L}} -3 \cdot 2^{-7}\) implies that arm \(k \in \mathcal{B}^m_\ell\) for all \(\ell \in \mathcal{L}_k\).
We also have \(\Delta_{k, \ell}^{m - 1} = \max\{2^{-3}, r_{\max, \ell}^{m - 1} - r_{k, \ell}^{m - 1} + 3 \cdot 2^{-7}\} \ge r_{\max, \ell}^{m - 1} - r_{k, \ell}^{m - 1} + 3 \cdot 2^{-7} \ge \frac{\Delta_{k, \ell}}{32}-3 \cdot 2^{-7}+ 3 \cdot 2^{-7}\).
% \begin{align*}
%     \Delta_k^{m-1} \ge \sqrt{\frac{L_k}{L_{\min}}}\left(r_{\max}^m - r_{k}^m\right) \ge \frac{\Delta_{k}}{32}\sqrt{\frac{L_k}{L_{\min}}} 
% \end{align*}
Therefore we can bound \(\mathcal{R}_{k, \ell}^m\) as:
\begin{align*}
    \Reg_{k, \ell}^m \le \Delta_{k, \ell} p_{k, \ell}^mT^m \le \Delta_{k, \ell} \cdot \frac{2^{-2m}(\Delta_{k, \ell}^{m-1})^{-2}}{K_\ell}\frac{L_{\min}}{L_k}\frac{K_\ell}{K}\cdot\lambda \frac{K}{L_{\min}} 2^{2(m - 1)} \le \frac{2^{6}\lambda}{\Delta_{k, \ell}L_k}
\end{align*}

Therefore

\begin{align}
    \Reg_{\mathcal{B}, 2} & \le 2^{6}\sum_{m = 1}^M \sum_{\ell \in \mathcal{L}}\sum_{k \in \mathcal{B}_\ell^m}\frac{\lambda}{L_k\Delta_{k, \ell}}\nonumber\\
     & = O\left(\log\left(\frac{8KL\log T}{\delta}\right)\sum_{m=1}^M\sum_{\ell \in \mathcal{L}}\sum_{k \in \mathcal{B}^m_\ell}\frac{1}{\Delta_{k, \ell}}\frac{1}{L_k}\right) \nonumber\\
     & \le O\left(\log\left(\frac{8KL\log T}{\delta}\right)\sum_{m=1}^M\sum_{k \in \mathcal{K}} \sum_{\substack{\ell \in \mathcal{L}\\ \text{s.t. } k \in \mathcal{B}^m_{\ell}}}\frac{1}{\Delta_{k, \ell}}\frac{1}{L_k}\right) \nonumber\\
     & = O\left(\log T \log\left(\frac{8KL\log T}{\delta}\right) \sum_{k \in \mathcal{K}}\frac{1}{\Delta_{k, \min}}\right)
     \end{align}
\paragraph{Case 3} \(\Delta_{k, \ell} > \frac{4}{2^m}\sqrt{\frac{L_{\min}}{L}}\) and \(\rho^{m - 1} \geq \Delta_{k}/32\). Therefore \(\Delta_{k, \ell} \le 32 \rho^{m - 1}\). Hence, if \(k \in \mathcal{B}_\ell^{m - 1}\), we have
\begin{align*}
    \Reg_{k, \ell}^m \le \Delta_{k, \ell}p_{k, \ell}^mT^m \le 32\rho^{m - 1}\cdot\frac{2^6}{K_\ell}\frac{L_{\min}}{L_k}\frac{K_\ell}{K}\cdot\lambda \frac{K}{L_{\min}}2^{2(m -1)}
    % {\lambda\Delta_{k}2^{2(m - 1)}} \le {8\frac{\lambda}{L_k}\rho^{m-1}2^{2m}}.
\end{align*}

Therefore
\begin{align}
    \Reg_{\mathcal{B}, 3} & \le 2^9\sum_{m = 1}^M\sum_{\ell \in \mathcal{L}}\sum_{k \in \mathcal{B}^m_\ell}\frac{\lambda}{L_k}\rho^{m-1}2^{2m} = 2^9\lambda\sum_{m = 1}^M\sum_{\ell \in \mathcal{L}}\sum_{k \in \mathcal{B}^m_\ell}\frac{1}{L_k}2^{2m}\rho^{m - 1}\nonumber \\
     & \overset{\eqref{eq:rho_MA}}= 2^9\lambda\sum_{m = 1}^M\sum_{\ell \in \mathcal{L}}\sum_{k \in \mathcal{B}^m_\ell}\frac{1}{L_k}2^{2m}\sum_{s = 1}^{m-1}\frac{2C^{s}}{8^{m - s - 1}T^s}\\
     &\le {2^9\lambda}\sum_{\ell \in \mathcal{L}}\sum_{k \in \mathcal{K}}\sum_{m = 1}^M\frac{1}{L_k}\sum_{s = 1}^{m-1}\frac{2^{2m + 1}C^{s}}{8^{m - s - 1}T^s}\nonumber\\
     & \overset{(a)} = 2^9\lambda\sum_{\ell \in \mathcal{L}}\sum_{k \in \mathcal{K}}\sum_{s = 1}^{M - 1}\frac{C^s}{L_k} \sum_{m = s + 1}^{M} \frac{2^{2m + 1}L_{\min}}{8^{m - 1 - s}\lambda K 2^{2(s - 1)}} \nonumber                               \\
     &\le 2^9\sum_{\ell \in \mathcal{L}}\sum_{k \in \mathcal{K}}\sum_{s = 1}^{M - 1}\frac{C^s}{K} \sum_{m = s+1}^{M} \frac{4^{m - 1 - s}}{8^{m - 1 - s}}\nonumber                                                \\
     & \le 2^9\sum_{\ell \in \mathcal{L}}\sum_{k \in \mathcal{K}}\sum_{s = 1}^{M - 1}\frac{C^s}{K} \le 2^9L\sum_{s = 1}^{M - 1}C^s\label{eq:case3_B},
\end{align}
where (a) is because \(T^m = \lambda \frac{K}{L_{\min}} 2^{2(m - 1)}\), and changing the order of summations.

On the other hand if \(k \in \mathcal{A}^m_\ell\), we have
\begin{align*}
    \Reg_{k, \ell}^m \le \Delta_{k, \ell}p_{k, \ell}^mT^m \le 32\rho^{m - 1}\cdot\frac{1}{|\mathcal{A}_\ell^m|}\cdot\lambda \frac{K}{L_{\min}}2^{2(m -1)}
    % {\lambda\Delta_{k}2^{2(m - 1)}} \le {8\frac{\lambda}{L_k}\rho^{m-1}2^{2m}}.
\end{align*}

Therefore following the same analysis leading to \eqref{eq:case3_B}, we have
\begin{align}
    \Reg_{\mathcal{A}, 3}
     & \le 2^3K\sum_{m = 1}^M\sum_{\ell \in \mathcal{L}}\sum_{k \in \mathcal{A}^m_\ell}\frac{2^{2m}}{|\mathcal{A}^m_\ell|}\rho^{m-1} \nonumber \\
     & \overset{\eqref{eq:rho_MA}}= 2^3K\sum_{\ell \in \mathcal{L}}\sum_{m = 1}^M \sum_{s = 1}^{m-1}\frac{2^{2m+1}C^{s}}{8^{(m - 1) - s - 1}T^s} \nonumber\\
     &= 2^3KL\sum_{s = 1}^{M - 1} C^s \sum_{m = s + 1}^{M} \frac{2^{2m}L_{\min}}{8^{m - 1 - s} K 2^{2(s - 1)}} \nonumber                                           \\
     & \overset{(a)}= L\sum_{s = 1}^{M - 1} C^s \sum_{m = s+1}^{M} \frac{4^{m - 5 - s}}{8^{m - 1 - s}} \le L\sum_{s = 1}^{M - 1} C^s\label{eq:case3_A},
\end{align}

Hence, the total aggregated regret can be bounded as follows, with probability at least \(1 - \delta\).
\begin{align}
    &\Reg_{\mathcal{A}, 1} + \Reg_{\mathcal{B}, 1} + \Reg_{\mathcal{B}, 2} + \Reg_{\mathcal{A}, 3} + \Reg_{\mathcal{B}, 3} \le O\left(LC
    + \log T \log(\frac{8KL\log T}{\delta})\frac{K}{\Delta_{\min}}\right).
\end{align}

\end{proof}

\end{document}